\documentclass{article}

\usepackage{microtype}
\usepackage{graphicx}
\usepackage{subcaption}
\usepackage{booktabs} 

\usepackage{hyperref}

\PassOptionsToPackage{sort}{natbib}

\makeatletter
\IfFileExists{algorithm.sty}{}{\@namedef{ver@algorithm.sty}{2026/11/10 unused dependency bypass}}
\IfFileExists{algorithmic.sty}{}{\@namedef{ver@algorithmic.sty}{2026/11/10 unused dependency bypass}}
\makeatother

\usepackage[accepted]{icml2026}

\usepackage{amsmath}
\usepackage{amssymb}
\usepackage{mathtools}
\usepackage{amsthm}
\usepackage{mathrsfs}
\usepackage{url} 

\usepackage{amsfonts}       
\usepackage{nicefrac}       
\usepackage{xcolor}         
\IfFileExists{ulem.sty}{\usepackage[normalem]{ulem}}{} 
\IfFileExists{wrapfig.sty}{\usepackage{wrapfig}}{}
\usepackage{enumerate}
\usepackage{bm}
\IfFileExists{bbm.sty}{\usepackage{bbm}}{}

\usepackage{times, enumitem}
\IfFileExists{upgreek.sty}{\usepackage{upgreek}}{}
\usepackage{caption}
\IfFileExists{ulem.sty}{\usepackage{ulem}}{}
\usepackage{multicol}

\usepackage{multicol}
\usepackage{multirow}
\usepackage{makecell}

\usepackage[capitalize,noabbrev]{cleveref}

\theoremstyle{plain}
\newtheorem{theorem}{Theorem}[section]
\newtheorem{proposition}[theorem]{Proposition}
\newtheorem{lemma}[theorem]{Lemma}
\newtheorem{corollary}[theorem]{Corollary}
\theoremstyle{definition}

\theoremstyle{remark}
\newtheorem{remark}[theorem]{Remark}

\newcommand{\R}{\mathbb{R}}

\newcommand{\E}{\mathbb{E}}

\newcommand{\Pb}{\mathbb{P}}

\DeclareMathOperator*{\argmin}{arg\,min}

\newcommand{\RN}[1]{%
	\textup{\uppercase\expandafter{\romannumeral#1}}
}

\theoremstyle{plain}
\newtheorem{technicallemma}{Technical Lemma}

\theoremstyle{remark}

\makeatletter
\newcommand{\customlabel}[2]{%
   \protected@write \@auxout {}{\string \newlabel {#1}{{#2}{\thepage}{#2}{#1}{}} }%
   \hypertarget{#1}{#2}
}
\makeatother

\usepackage[textsize=tiny]{todonotes}

\icmltitlerunning{Semi-Supervised Learning with Noisy Proxy Covariates}

\hypersetup{colorlinks=false,pdfborder={0 0 0},
pdftitle={Semi-Supervised Learning with Noisy Proxy Covariates: Generalization Bounds and Distribution Regression},
pdfauthor={Kwangho Kim and Jisu Kim},
pdfsubject={Semi-Supervised Learning with Noisy Proxy Covariates},
pdfkeywords={semi-supervised learning, noisy covariates, distribution regression}}

\begin{document}

\twocolumn[
  \icmltitle{Semi-Supervised Learning with Noisy Proxy Covariates: \\Generalization Bounds and Distribution Regression}
 



  \begin{icmlauthorlist}
    \icmlauthor{Kwangho Kim}{ku}
    \icmlauthor{Jisu Kim}{snu}
  \end{icmlauthorlist}

  \icmlaffiliation{ku}{Department of Statistics, Korea University, Seoul, Korea}
  \icmlaffiliation{snu}{Department of Statistics, Seoul National University, Seoul, Korea}
  \icmlcorrespondingauthor{Jisu Kim}{jkim82133@snu.ac.kr}

  \icmlkeywords{Machine Learning, ICML}

  \vskip 0.3in
]



\printAffiliationsAndNotice{}  

\begin{abstract}
    In many modern machine learning pipelines, abundant pretrained representations serve as noisy proxy covariates, while task-specific labels remain scarce. We study semi-supervised regression in this setting, and propose a simple two stage estimator that learns kernel eigenfeatures from all proxy covariates and fits a ridge predictor on labeled data. We derive finite sample bounds showing that fast labeled sample rates are recovered when proxy perturbation is controlled and unlabeled proxy covariates are sufficiently abundant. We also show that distribution regression is a direct special case, with analogous guarantees when the finite bag size is large enough. Experiments show consistent gains over supervised and semi-supervised baselines, especially in low label regimes.
\end{abstract}

\section{Introduction}

Semi-supervised learning (SSL) seeks provable gains from combining a small labeled sample with a large unlabeled sample \citep{Zhu09, van2020survey}. Classical theory typically explains such gains through structure in the marginal covariate distribution, most notably the manifold and cluster (low-density separation) assumptions \citep{zhu2003semi, belkin2004semi,Belkin06, Rigollet06,Lafferty07, Belkin08, Niyogi13a,  Singh08, Sinha09}. While these assumptions yield fast rates in idealized regimes, they are hard to validate from limited labels, and prior theory shows that unlabeled data need not improve performance unless the marginal covariate distribution is informative about the regression or decision function. Recent work also sharpens the picture for graph-based regularization by characterizing when Laplacian-based SSL does and does not help \citep{cabannes2021laplacian}.

This paper studies a complementary obstacle that is ubiquitous in modern machine learning pipelines: the covariates available to the learner are often noisy proxies. In contemporary systems, including foundation-model pipelines, downstream predictors are trained on representations produced from massive unlabeled corpora, but these representations are imperfect, domain-dependent measurements of latent variables. Empirically, data curation and distributional mismatch can dominate performance even when unlabeled scale is enormous \citep{gadre2023datacomp}. From a statistical viewpoint, this motivates SSL with noisy covariates: we observe labeled pairs $(\widehat{X}_i,Y_i)_{i=1}^n$ and additional unlabeled covariates $\widehat{X}_{n+1:N}$, where $\widehat{X}$ is a contaminated or proxy observation for an unobserved latent covariate $X$. The basic question is whether, and under what conditions, unlabeled proxy covariates improve generalization despite measurement error. Relatedly, recent theory on learning with noisy features underscores that “more data” and “more features” can interact in nonclassical ways even in supervised settings \citep{li2020benign}, reinforcing the need for SSL analyses that explicitly model covariate noise.

Our approach is based on spectral methods and reproducing kernel Hilbert spaces. Using the unlabeled proxy covariates, we estimate leading eigenspaces of an empirical integral operator and fit the labeled regression on the resulting eigenfeatures, following the kernel projection viewpoint \citep{Ji12, Guo12}. The technical core is a generalization analysis that tracks (i) approximation error from projecting onto a finite eigenspace, (ii) operator and eigenspace estimation error induced by finite \(N\) and proxy perturbation, (iii) coefficient estimation error from \(n\) labels, and (iv) ridge regularization. The bounds identify when sufficiently accurate unlabeled proxy covariates recover the oracle feature rate under polynomial eigenvalue decay.

\subsection{Measurement Error and Distribution Regression}

Measurement error (or errors-in-variables) models replace latent covariates $X$ by noisy observations $\widehat{X}$; while classical work is extensive, nonparametric measurement error with unknown error mechanisms remains challenging, and SSL versions of such problems are largely unexplored \citep{carroll2009nonparametric, schennach2016recent}. Our framework treats SSL with noisy covariates as the primary object.

Distribution regression is a key special case, in which each covariate is a probability distribution $P$, observed only through a finite bag of samples, and one predicts $Y$ from $P$ \citep{Poczos13, Oliva13, Szabo16}. This fits our framework by viewing $\widehat{P}$, constructed from finitely many samples, as a proxy for $P$. We develop theory for this case as well. Recent work on kernel based representations of distributions \citep{kachaiev2024mke} further underscores the role of distribution valued covariates as a modern learning primitive.

\subsection{Main results and contributions}

We analyze a kernel projection semi-supervised estimator that first learns spectral features from all proxy covariates \(\widehat X_{1:N}\) and then fits a regularized linear predictor using the labeled pairs \((\widehat X_i,Y_i)_{i=1}^n\). The main result is a finite sample excess risk bound that separates spectral truncation, proxy based operator error, labeled estimation, and ridge regularization. Under the stronger two sided polynomial decay condition \(\lambda_s\asymp s^{-q}\), the bound yields the oracle feature rate \(n^{-q/(q+1)}\) whenever the operator learning and proxy perturbation terms are negligible at the oracle truncation level. In the labeled only regime, when \(N=n\), the additional operator learning term remains and its effect depends explicitly on the local eigengap sequence. We also show that distribution regression is a direct special case: the latent covariate is a distribution \(P\) observed through a finite sample proxy \(\widehat P\), and the generic proxy perturbation term becomes a bag size dependent error envelope.

Our contributions are fourfold. First, we develop finite sample generalization bounds for semi-supervised regression with noisy proxy covariates, allowing both well specified and misspecified regression functions and explicitly tracking the proxy perturbation floor. Second, the bounds characterize when unlabeled proxy covariates make the eigengap-dependent operator-estimation term negligible, with explicit dependence on total proxy-covariate sample size, local eigengap, proxy quality, and ridge regularization. Third, we show that distribution regression fits the same framework by viewing the empirical distribution \(\widehat P\) as a proxy for a latent distribution \(P\). Fourth, we provide conditional guarantees for validation based tuning and approximate eigendecomposition, under explicit validation and subspace approximation conditions.

\section{Set Up and Notation}
\label{sec:setup}

Our goal is to establish a unified framework for semi-supervised regression with noisy proxy covariates. Let \(\mathcal U\) be a compact metric space with metric \(d\), and let \((U,Y)\sim\Pb\), where \(U\in\mathcal U\) and \(Y\in\R\). We observe \(N\) covariate instances, of which only the first \(n\) are labeled and the remaining \(N-n\) are unlabeled, with \(n\le N\). For labeled observations, the latent regression model is
\begin{equation}\label{eqn:reg-problem}
Y_i=f(U_i)+\mu_i,\qquad i=1,\ldots,n,
\end{equation}
where \(f:\mathcal U\to\R\) is unknown and the noise variables \(\mu_i\) are independent, mean zero, and sub-Gaussian. We assume iid latent--proxy--response triples, \(f\in L_2(\Pb_U)\), and existence of \(g\) below. Test risks hold the fitted function fixed. We write \(\tau_0^2=\E[\mu_i^2]\). Throughout, we assume the proxy mechanism does not introduce outcome-noise bias, so that
\[
\E[\mu_i\mid U_i,\widehat U_i]=0 .
\]
This holds, for example, when the proxy \(\widehat U_i\) is generated conditionally independently of \(\mu_i\) given \(U_i\).

The latent covariate \(U_i\) is not observed. Instead, we observe proxy covariates
\[
\widehat U_i\sim \mathsf M(\cdot\mid U_i),\qquad i=1,\ldots,N,
\]
generated from a possibly unknown measurement channel \(\mathsf M\), with \(\widehat U_i\in\mathcal U\). The observed sample is
\[
\mathcal D_{\mathcal U}^{\mathrm{obs}}
=
\{(\widehat U_i,Y_i)\}_{i=1}^n
\cup
\{\widehat U_i\}_{i=n+1}^N .
\]
This formulation covers settings where downstream predictors use noisy proxies of latent features, such as representations produced by large scale pretraining.

We next introduce the hypothesis space and target predictor. Let \(\mathcal H_\kappa\) be a reproducing kernel Hilbert space (RKHS) on \(\mathcal U\) with bounded Mercer kernel \(\kappa:\mathcal U\times\mathcal U\to\R\). Define the best-in-class predictor
\begin{subequations}\label{eq:def-g-eps}
\begin{align}
g
&\in \argmin_{h\in\mathcal H_\kappa}
\E\!\left[(f(U)-h(U))^2\right], \label{eq:g-reg}\\
&\in \argmin_{h\in\mathcal H_\kappa}
\E\!\left[(Y-h(U))^2\right], \label{eq:g-pred}\\
\epsilon^2
&=
\min_{h\in\mathcal H_\kappa}
\E\!\left[(f(U)-h(U))^2\right]. \label{eq:err-reg}
\end{align}
\end{subequations}
The two minimization problems in \eqref{eq:def-g-eps} have the same solution because \(Y=f(U)+\mu\) and \(\E[\mu\mid U]=0\). Thus \(g\) is the best-in-class latent predictor, and \(\epsilon^2\) is the approximation error; \(\epsilon^2=0\) in the well-specified case. In the misspecified case, our analysis requires that the
best-in-class minimizer exists in $\mathcal H_\kappa$ and has bounded RKHS norm, as specified in Assumption~\ref{assumption:a6}.

For \(h\in\mathcal H_\kappa\), define the population kernel integral operator
\[
\mathscr L(h)(\cdot)
=
\E\!\left[\kappa(U,\cdot)h(U)\right].
\]
Given samples, define the latent empirical operator and the proxy empirical operator by
\begin{align*}
\mathscr L_N(h)(\cdot)
&=
\frac{1}{N}\sum_{i=1}^N
\kappa(U_i,\cdot)h(U_i),\\
\widehat{\mathscr L}_N(h)(\cdot)
&=
\frac{1}{N}\sum_{i=1}^N
\kappa(\widehat U_i,\cdot)h(\widehat U_i).
\end{align*}
Our analysis controls \(\|\widehat{\mathscr L}_N-\mathscr L\|\), which captures both finite sample operator error and proxy perturbation.

Let \(\{(\widehat\lambda_j,\widehat\phi_j)\}_{j\ge1}\) denote the eigenpairs of \(\widehat{\mathscr L}_N\), ordered by nonincreasing eigenvalues, with eigenfunctions normalized in \(\mathcal H_\kappa\). For \(s\ge1\), the estimator uses the leading eigenfunctions \(\{\widehat\phi_j\}_{j=1}^s\). We write \(\|\cdot\|_{\mathcal H_\kappa}\) for the RKHS norm. Population eigenfunctions used in the analysis are normalized in \(L_2(\Pb_U)\), and the relevant norm is written explicitly when needed.

The framework specializes to the two settings studied in this paper. First, for noisy Euclidean covariates analyzed in Section~\ref{sec:me_ssl}, we take \(\mathcal U=\mathcal X\subset\R^d\), set \(U=X\), and let \(\widehat U=\widehat X\) be a noisy proxy of \(X\). Second, for distribution regression, we take \(\mathcal U=\mathcal P\), a metric space of probability measures with \(d=D\), as specified in Section~\ref{sec:distribution-reg}, set \(U=P\), and let \(\widehat U=\widehat P\) be a finite-sample estimator of \(P\). In both cases, the observed semi-supervised sample has the form \(\mathcal D_{\mathcal U}^{\mathrm{obs}}\), so the same algorithm and analysis apply after instantiating \((\mathcal U,d,\kappa,\mathsf M)\).

\section{SSL with Noisy Proxy Covariates}
\label{sec:me_ssl}

We study semi-supervised regression when the latent covariate $X$ is not observed and only a noisy proxy $\widehat X$ is available. The goal is to use the full proxy sample $\{\widehat X_i\}_{i=1}^N$ (labeled and unlabeled) to learn a low-dimensional spectral representation of the covariate geometry, and then fit the best-in-class predictor using the labeled responses.

\subsection{Estimation}
We now describe the estimation procedure explicitly. The method has two stages:
(i) learn a data-dependent feature map from all proxy covariates, and
(ii) fit a ridge-type regressor on the labeled data using these features.

\textbf{Stage 1: spectral feature extraction from proxy covariates.}
Let $\mathcal H_\kappa$ be an RKHS on $\mathcal X$ with bounded Mercer kernel
$\kappa:\mathcal X\times\mathcal X\to\R$. Define the population operator
\begin{equation}\label{eq:L-pop}
\mathscr L(h)(\cdot)
=
\E\!\big[\kappa(X,\cdot)\,h(X)\big],
\quad
h\in\mathcal H_\kappa,
\end{equation}
and the proxy empirical operator based on the observed covariates
$\widehat X_1,\ldots,\widehat X_N$,
\begin{equation}\label{eq:L-proxy}
\widehat{\mathscr L}_N(h)(\cdot)
=
\frac{1}{N}\sum_{i=1}^N
\kappa(\widehat X_i,\cdot)\,h(\widehat X_i).
\end{equation}
Let $\{(\hat\lambda_j,\hat\phi_j)\}_{j\ge1}$ be the eigenpairs of $\widehat{\mathscr L}_N$ ordered by nonincreasing
$\hat\lambda_j$, with $\{\hat\phi_j\}$ orthonormal in $\mathcal H_\kappa$. For a chosen dimension $s\ge1$, define the
feature map
\[
\hat\Phi_s(x)=(\hat\phi_1(x),\ldots,\hat\phi_s(x))^\top\in\R^s .
\]

\textbf{Stage 2: ridge regression using labeled examples.}
Using only the labeled pairs $\{(\widehat X_i,Y_i)\}_{i=1}^n$, fit
\begin{equation}\label{eq:ridge-features}
\begin{gathered}
\hat\gamma \in \argmin_{\gamma\in\R^s}
\left\{
\frac{1}{n}\sum_{i=1}^n
\big(Y_i-\langle\gamma,\hat\Phi_s(\widehat X_i)\rangle\big)^2
+\xi\|\gamma\|_2^2
\right\},\\
\hat g(\cdot)= \langle\hat\gamma,\hat\Phi_s(\cdot)\rangle,
\end{gathered}
\end{equation}
where $\xi\ge0$ is a regularization parameter, and $\hat g$ denotes the fitted predictor. Thus, unlabeled and labeled proxy covariates jointly determine the feature
map $\hat\Phi_s$, while labels enter only through the ridge fit for $\hat\gamma$.

\paragraph{Kernel-matrix form.}
For implementation, it is convenient to express the eigenfunctions of $\widehat{\mathscr L}_N$ in terms of the finite kernel matrix
constructed from the proxy covariates.
Let $\mathbf K\in\R^{N\times N}$ with entries $\mathbf K_{ij}=\kappa(\widehat X_i,\widehat X_j)$. If $(v_j,\sigma_j)$ are
the leading eigenpairs of $\mathbf K$ (descending $\sigma_j$), then
$\hat\lambda_j=\sigma_j/N$ and
\begin{equation}\label{eq:eig-expansion}
\hat\phi_j(\cdot)
=
\sigma_j^{-1/2}\sum_{i=1}^N (v_j)_i\,\kappa(\widehat X_i,\cdot).
\end{equation}
In particular, $\hat\Phi_s(\widehat X_i)$ is obtained by evaluating \eqref{eq:eig-expansion} at $\widehat X_i$, and
\eqref{eq:ridge-features} becomes ordinary ridge regression in $\R^s$ with design vectors
$\{\hat\Phi_s(\widehat X_i)\}_{i=1}^n$.

\subsection{Generalization Error Bounds}

We first state a minimal set of assumptions under which we derive finite-sample generalization bounds and fast rates.
\begin{itemize}[leftmargin=*]
\item \customlabel{assumption:a1}{(A1)} \emph{Sub-Gaussian noise and controlled measurement error.}
The regression noises satisfy a uniform conditional sub-Gaussian bound given \((U_i,\widehat U_i)\), as specified in Appendix~\ref{app:technical_conditions}.
Moreover, there exist constants $C_D,\theta>0$ such that, for every
$k\ge1$,
\[
\left\{
\E\!\left[
D_2(X,\widehat X)^k
\mid X
\right]
\right\}^{1/k}
\le
C_D\theta\sqrt{k}.
\]

\item \customlabel{assumption:a2}{(A2)} \emph{Bounded H\"older kernel.}
The kernel is bounded, $|\kappa(x,x)|\le 1$, and $\kappa(x,\cdot)$ is H\"older in $\mathcal H_\kappa$:
there exist $L_\kappa>0$ and $0<\beta_\kappa\le 1$ such that
\[
\|\kappa(\cdot,x)-\kappa(\cdot,x')\|_{\mathcal H_\kappa}
\le
L_\kappa\,D_2(x,x')^{\beta_\kappa},
\quad \forall x,x'\in\mathcal X.
\]

\item \customlabel{assumption:a3}{(A3)} \emph{Polynomial eigenvalue decay.}
Let $\{\lambda_j\}_{j\ge1}$ be the eigenvalues of $\mathscr L$ in \eqref{eq:L-pop}.
There exist $a>0$ and $q>2$ such that $\lambda_j\le a^2 j^{-q}$ for all $j\ge1$.

\item \customlabel{assumption:a4}{(A4)} \emph{Regular eigenfunctions.}
Let $\{\phi_j\}_{j\ge1}$ be the eigenfunctions of $\mathscr L$.
There exists $M_\phi<\infty$ such that $\sup_{x\in\mathcal X}\sup_{j\ge1}|\phi_j(x)|\le M_\phi$.
Moreover each $\phi_j$ is H\"older in $x$ with exponent $0<\beta_\phi\le 1$ and constants $\{L_{\phi_j}\}$ satisfying
\[
C_\phi^2
:=
\sum_{j\ge1}\lambda_j L_{\phi_j}^2
<
\infty.
\]

\item \customlabel{assumption:a5}{(A5)} \emph{Sample size and eigengap condition.}
Let
\begin{equation}\label{eq:gap-def}
\Delta_s = \lambda_s-\lambda_{s+1}.
\end{equation}
Assume $\Delta_s>0$ and that $n$ and $N$ are large enough so that, with high probability,
\begin{equation}\label{eq:gap_event}
\|\widehat{\mathscr L}_N-\mathscr L\|_{\mathrm{op}} \le \Delta_s/3,
\end{equation}
which guarantees separation of the leading $s$-dimensional empirical and population spectral subspaces. Let $\Pi_s$ and $\widehat\Pi_s^{\mathcal H}$ denote the $\mathcal H_\kappa$-orthogonal projectors onto these two subspaces. Lemma~\ref{lem:eigenspace-bound} establishes
\[
\|\widehat\Pi_s^{\mathcal H}-\Pi_s\|_{\mathrm{op}}
\lesssim
\frac{\|\widehat{\mathscr L}_N-\mathscr L\|_{\mathrm{op}}}{\Delta_s}.
\]

\item \customlabel{assumption:a6}{(A6)} \emph{Target regularity.}
For some fixed constant $R<\infty$, the best-in-class predictor satisfies
\[
\|g\|_{\mathcal H_\kappa}\le R.
\]
\end{itemize}

Assumption~\ref{assumption:a1} controls the proxy error $D_2(\widehat X,X)$ without requiring an additive measurement-error model. Assumption~\ref{assumption:a2} is standard in kernel learning: boundedness holds after rescaling, and H\"older continuity is satisfied by many kernels on compact domains. Assumption~\ref{assumption:a3} is a spectral-decay condition widely used in kernel ridge regression and spectral approximation; larger $q$ corresponds to faster eigen-decay and hence lower effective dimension. Assumption~\ref{assumption:a4} imposes bounded and H\"older eigenfunctions, ensuring stable evaluation and perturbation control. Assumption~\ref{assumption:a5} is the usual eigengap regime for identifying the leading $s$-dimensional subspace: \eqref{eq:gap_event} is exactly the condition under which Davis--Kahan perturbation bounds apply \citep{yu2015useful}. Throughout, we write $r_s$ for any deterministic lower bound on the eigengap, $0<r_s\le \Delta_s$; one may take $r_s=\Delta_s$. Assumption~\ref{assumption:a6} fixes the RKHS radius of the best-in-class predictor and is used to control spectral truncation, subspace perturbation, and the transfer from proxy evaluation at $\widehat X$ to latent evaluation at $X$. Overall, these assumptions align with those in prior spectral semi-supervised learning and kernel learning theory (e.g., \citealt{Ji12, Guo12, Koltchinskii10, Szabo16}), with Assumption~\ref{assumption:a1} capturing the additional proxy-covariate setting.

Examples of kernels and proxy models satisfying these regularity conditions are discussed in Appendix~\ref{app:regularity_examples}.

\begin{remark}[Choice of truncation level]
Assumption~\ref{assumption:a5} is stated for a fixed truncation level \(s\) and the corresponding eigengap \(\Delta_s\).
The same perturbation argument can be extended to data adaptive choices of \(s\), provided the selected index satisfies a corresponding separation condition.
For example, one may choose \(s\) from a candidate set and use \(r_s\) as a deterministic lower bound on the relevant local eigengap.
For clarity, we state the main results for a fixed \(s\) and the standard gap \(\Delta_s\).
\end{remark}

Next, we analyze performance relative to the best-in-class predictor in the chosen RKHS. This target is defined on the latent covariate $X$, whereas the estimator is constructed from the observed proxy covariates $\widehat X$. The proxy perturbation is accounted for explicitly in the bound below. Define 
\begin{align*}
g&\in\argmin_{h\in\mathcal H_\kappa}\E\!\left[(Y-h(X))^2\right],\\
\epsilon^2 &= \min_{h\in\mathcal H_\kappa} \E\!\left[(f(X)-h(X))^2\right]. 
\end{align*} 
Since \(Y=f(X)+\mu\) with \(\E[\mu\mid X]=0\), the cross term vanishes and
\[
\E\!\left[(g(X)-Y)^2\right]
=
\E\!\left[(g(X)-f(X))^2\right]+\tau_0^2
=
\epsilon^2+\tau_0^2 .
\]
Write \(\mathcal E(\widehat g)=\E[(\widehat g(X)-g(X))^2]\). Below, \(R_{s,n,N}^2\) denotes a deterministic rate expression. All proofs are in the appendix.

\begin{theorem}\label{prop:error-estimator-best}
Under Assumptions~\ref{assumption:a1}--\ref{assumption:a6} and the sampling conditions above,
\[
\mathcal E(\widehat g)=O_P(R_{s,n,N}^2),
\]
where, for any eigengap lower bound \(0<r_s\le \Delta_s\),
\begin{equation}\label{eq:R_bound}
R_{s,n,N}^2
:=
s^{-q}
+
\frac{1}{r_s^2}
\left\{
\frac{1}{N}+C_{\beta',\theta}
\right\}
+
\frac{s}{n}
+
\xi.
\end{equation}
Here \(\beta'=\min\{\beta_\kappa,\beta_\phi\}\).
\end{theorem}

Equation~\eqref{eq:R_bound} makes the roles of $s$ (spectral truncation), $N$ (total proxy-covariate sample size used for operator learning), and $n$ (labeled sample size for coefficient learning) explicit. Combining Theorem~\ref{prop:error-estimator-best} with elementary $L_2$ arguments yields corresponding bounds for the regression and prediction risks.

\begin{theorem}[Regression error]\label{thm:error-regression}
Under the conditions of Theorem~\ref{prop:error-estimator-best},
\begin{equation}\label{eq:reg_bound}
\E\!\left[(\widehat g(X)-f(X))^2\right]
=
\epsilon^2+O_P(R_{s,n,N}^2),
\end{equation}
where $R_{s,n,N}$ is defined in
Theorem~\ref{prop:error-estimator-best}.
\end{theorem}

\begin{theorem}[Prediction error]\label{thm:error-prediction}
Under the conditions of Theorem~\ref{thm:error-regression},
\begin{equation}\label{eq:pred_bound}
\E\!\left[(\widehat g(\widehat X)-Y)^2\right]
=
\epsilon^2+\tau_0^2+O_P(R_{s,n,N}^2),
\end{equation}
including the proxy evaluation transfer term in the stated order.
\end{theorem}

\subsection{Optimal tuning and rates}
\label{sec:rates}

We next discuss the rate implications of Theorem~\ref{prop:error-estimator-best}. Under $\lambda_s\asymp s^{-q}$, once the ridge regularization term in \eqref{eq:R_bound} is negligible, the oracle spectral tradeoff between $s^{-q}$ and $s/n$ yields a fast rate whenever the operator learning and proxy perturbation term is negligible at the oracle truncation level, as formalized below.

\begin{corollary}\label{cor:error-rate}
Suppose the conditions of Theorem~\ref{prop:error-estimator-best} hold and $\lambda_s\asymp s^{-q}$. If, at $s\asymp n^{1/(q+1)}$,
\begin{equation}\label{eq:negligible_operator_condition}
\frac{1}{r_s^2}
\left\{
\frac{1}{N}+C_{\beta',\theta}
\right\}
=
o\!\left(s^{-q}+\frac{s}{n}\right)
\end{equation}
and
\begin{equation}\label{eq:negligible_ridge_condition}
\xi
=
o\!\left(
s^{-q}+\frac{s}{n}
\right).
\end{equation}
hold, then choosing $s\asymp n^{1/(q+1)}$ yields
\[
R_{s,n,N}^2
=
O\!\left(n^{-q/(q+1)}\right).
\]
Consequently, Theorems~\ref{thm:error-regression} and \ref{thm:error-prediction} achieve the same rate for the excess terms beyond $\epsilon^2$ and $\epsilon^2+\tau_0^2$, respectively.
\end{corollary}

Condition~\eqref{eq:negligible_ridge_condition} permits \(\xi=0\) when the fit is stable. Corollary~\ref{cor:error-rate} requires decreasing operator and proxy errors. A fixed nonzero proxy scale need not satisfy the sampling conditions of Theorem~\ref{prop:error-estimator-best}; the corresponding finite-design interpretation is given in Appendix~\ref{app:fixed_proxy}.

To make the dependence on the eigengap explicit, suppose \(r_s\gtrsim s^{-\alpha}\) for some \(\alpha\ge0\). At \(s\asymp n^{1/(q+1)}\), it suffices that
\[
N\gg n^{(q+2\alpha)/(q+1)},\qquad
C_{\beta',\theta}=o\!\left(n^{-(q+2\alpha)/(q+1)}\right).
\]
Under the regular-spacing heuristic \(r_s\asymp\Delta_s\asymp s^{-(q+1)}\), this becomes \(N\gg n^{(3q+2)/(q+1)}\). The derivation is given in Appendix~\ref{app:gap_rates}.

In the labeled-only regime \(N=n\), with negligible proxy and ridge terms, the displayed upper bound becomes
\[
s^{-q}+s^{2\alpha}/n+s/n.
\]
Formal balancing gives the polynomial exponent \(q/(q+\max\{1,2\alpha\})\). Applying the sampling conditions can require a slowly varying correction; Appendix~\ref{app:labeled_only} states a sufficient truncation choice.

Thus sufficiently many accurate unlabeled proxy covariates
remove the eigengap-dependent operator-estimation term
from the present upper bound and recover the oracle feature
rate $n^{-q/(q+1)}$. We do not claim that the slower
labeled-only rate is intrinsic; it reflects the
Davis--Kahan perturbation argument used in the present
analysis.

\subsection{Practical considerations}
\label{subsec:practical}

This subsection states two consequences of the preceding theory that make the method implementable in practice: data-driven selection of the tuning parameters $(s,\xi)$ and stable use of approximate eigendecompositions.

\subsubsection{Data-driven tuning of $s$ and $\xi$}
\label{subsubsec:tuning}

The oracle choices of $(s,\xi)$ in Corollary~\ref{cor:error-rate} depend on unknown spectral, eigengap, and proxy perturbation quantities. We therefore choose $(s,\xi)$ by validation while still using all proxy covariates, labeled and unlabeled, to construct the eigenfeatures.

Split the labeled indices $\{1,\ldots,n\}$ into disjoint training and validation sets $\mathcal I_{\mathrm{tr}}$ and $\mathcal I_{\mathrm{va}}$, with sizes $n_{\mathrm{tr}}$ and $n_{\mathrm{va}}$. Using all proxy covariates $\{\widehat X_i\}_{i=1}^N$, compute $\widehat\Phi_s(\cdot)$ for $s\in\mathcal S$. For each $(s,\xi)\in\mathcal S\times\Xi$, fit $\widehat g_{s,\xi}$ by \eqref{eq:ridge-features} using only $\mathcal I_{\mathrm{tr}}$, and choose
\[
\begin{gathered}
(\widehat s,\widehat\xi)\in
\argmin_{(s,\xi)\in\mathcal S\times\Xi}
\frac{1}{n_{\mathrm{va}}}
\sum_{i\in\mathcal I_{\mathrm{va}}}
\left\{Y_i-\widehat g_{s,\xi}(\widehat X_i)\right\}^2,\\
\widehat g_{\mathrm{sel}}=\widehat g_{\widehat s,\widehat\xi}.
\end{gathered}
\]
The following theorem shows that this selector satisfies a finite-grid oracle inequality. We state it under the same stability and proxy-transfer conditions used in the proof of Theorem~\ref{prop:error-estimator-best}, uniformly over the candidate grid.

\begin{theorem}\label{thm:cv_rate}
Assume the training conclusions of Theorem~\ref{prop:error-estimator-best} hold uniformly over the candidate grid, and the validation condition in Appendix~\ref{app:validation_conditions}. Let \(M=|\mathcal S||\Xi|\). Conditional on the training good event, with probability at least \(1-\delta\),
\[
\begin{aligned}
\mathcal E(\widehat g_{\mathrm{sel}})
&\lesssim\inf_{(s,\xi)\in\mathcal S\times\Xi}R_{s,n_{\mathrm{tr}},N}^2\\
&\quad+\Gamma_{\mathcal G}
+C_{\rm va}\frac{\log(M/\delta)}{n_{\rm va}}.
\end{aligned}
\]
Here \(\Gamma_{\mathcal G}\) is the uniform proxy-transfer cost defined there. If \(n_{\rm tr}\asymp n_{\rm va}\asymp n\) and a polynomial-size grid contains an oracle-rate candidate, the selected predictor has the oracle rate if \(C_{\rm va}\) is bounded and \(\Gamma_{\mathcal G}=O_P(n^{-q/(q+1)})\).
\end{theorem}

For eigenfeatures reusing validation covariates, the concentration condition remains an additional assumption; Appendix~\ref{app:validation_conditions} identifies the remaining verification step.

\subsubsection{Approximate eigendecomposition}
\label{subsubsec:approx_eig}

Computing the leading eigenpairs of the $N\times N$ kernel matrix can be a bottleneck when $N$ is large. In practice, one often constructs approximate eigenfeatures using methods such as Nystr\"om or randomized eigensolvers \citep[e.g.,][]{drineas2005nystrom, mahoney2011randomized}. We next show that the guarantees are stable under sufficiently accurate subspace approximation.

Let $\widetilde\Phi_s(\cdot)\in\R^s$ be an approximate feature map for the top-$s$ proxy eigenspace, and let $\widetilde g$ be the ridge estimator obtained by replacing $\widehat\Phi_s$ with $\widetilde\Phi_s$ in \eqref{eq:ridge-features}. Let $\widehat\Pi_s^{\mathcal H}$ and $\widetilde\Pi_s^{\mathcal H}$ denote the $\mathcal H_\kappa$-orthogonal projectors onto $\mathrm{span}\{\widehat\phi_1,\ldots,\widehat\phi_s\}$ and $\mathrm{span}\{\widetilde\phi_1,\ldots,\widetilde\phi_s\}$, respectively. We impose the following subspace accuracy condition on the approximate eigenspace.

\begin{itemize}[leftmargin=*]
\item \customlabel{assump:spec_err}{(A7)} \emph{Approximate eigenspace accuracy.}
There exists $\varepsilon_{\mathrm{spec}}\ge 0$ such that $\|\widetilde\Pi_s^{\mathcal H}-\widehat\Pi_s^{\mathcal H}\|_{\mathrm{op}}
\le
\varepsilon_{\mathrm{spec}}$.
\end{itemize}

Applying \ref{assump:spec_err} to a solver requires an RKHS lifting and normalization of its output. An operator approximation on the same RKHS yields a projector bound through the empirical eigengap; a matrix error bound alone does not specify that lifting. Appendix~\ref{app:approx_conditions} states the required relation. The following theorem establishes that sufficiently accurate approximate eigenfeatures preserve the generalization rate.

\begin{theorem}\label{thm:approx_eig}
Assume the conditions of Theorem~\ref{prop:error-estimator-best}, Assumption~\ref{assump:spec_err}, and the RKHS-orthonormal feature construction with independent numerical randomness in Appendix~\ref{app:approx_conditions}, with \(\varepsilon_{\rm spec}^2/\lambda_s\to0\). Then
\[
\E\!\left[(\widetilde g(X)-g(X))^2\right]
=O_P\!\left(R_{s,n,N}^2+\varepsilon_{\mathrm{spec}}^2\right),
\]
Consequently, under the conditions of Corollary~\ref{cor:error-rate}, if $s\asymp n^{1/(q+1)}$ and $\varepsilon_{\mathrm{spec}}=o\{n^{-q/(2q+2)}\}$, then $\widetilde g$ attains the same rate $n^{-q/(q+1)}$.
\end{theorem}

Theorem~\ref{thm:approx_eig} shows that approximate eigenfeature construction adds only a subspace-approximation term. Thus, the statistical rate is unchanged whenever $\varepsilon_{\mathrm{spec}}^2$ is smaller than the target statistical error.

In sum, Theorem~\ref{thm:cv_rate} shows that the rate in Corollary~\ref{cor:error-rate} is attainable without oracle knowledge of $(s,\xi)$ when the grid contains near-oracle values and the additional validation conditions hold, while Theorem~\ref{thm:approx_eig} shows that scalable eigensolvers preserve the same rate provided the subspace approximation error is sufficiently small.

\section{Application to Distribution Regression}
\label{sec:distribution-reg}

This section applies the general framework of Section~\ref{sec:me_ssl} to distribution regression, where each covariate is a probability distribution observed only through finitely many samples. The latent covariate is a distribution \(P\), and the observed covariate is a proxy \(\widehat P\) constructed from a finite bag of draws from \(P\). Thus distribution regression is a structured noisy-covariate problem. The analysis below is a direct specialization of Section~\ref{sec:me_ssl}, with the bag size \(m\) controlling the proxy perturbation level.

\subsection{Setup}

Let \(\mathcal Z\subset\R^d\) be a compact domain and let \(\mathcal P\) be a set of probability measures on \(\mathcal Z\). We observe \(N\) distribution-valued covariates \(\{P_i\}_{i=1}^N\), of which only the first \(n\) have labels. The labeled sample satisfies
\begin{equation}\label{eq:dist_reg_model}
Y_i=f(P_i)+\mu_i,\quad i=1,\ldots,n,
\end{equation}
where \(f:\mathcal P\to\R\) is an unknown regression functional and \(\{\mu_i\}\) are independent, mean-zero, sub-Gaussian noises with \(\tau_0^2=\E[\mu_i^2]\). Each distribution \(P_i\) is not directly observed. Instead, for each \(i\in\{1,\ldots,N\}\), we observe a bag of samples
\begin{equation}\label{eq:bag_samples}
Z_{i1},\ldots,Z_{im_i}\sim P_i \quad \text{i.i.d.},
\end{equation}
and construct a proxy \(\widehat P_i\) from the bag \(\overline{\mathcal Z}_i=\{Z_{i1},\ldots,Z_{im_i}\}\). For simplicity, we take \(m_i=m\) for all \(i\); heterogeneous bag sizes can be handled by replacing \(m\) with an effective bag size.

The observed semi-supervised data are
\[
\mathcal D_{\mathcal P}^{\mathrm{obs}}
=
\{(\widehat P_i,Y_i)\}_{i=1}^n \cup \{\widehat P_i\}_{i=n+1}^N,
\]
which has the same structure as Section~\ref{sec:me_ssl} under the identification \(X\equiv P\) and \(\widehat X\equiv\widehat P\).

Let \(\kappa:\mathcal P\times\mathcal P\to\R\) be a bounded positive definite kernel on distributions and let \(\mathcal H_\kappa\) be its RKHS. Examples include kernels induced by distribution embeddings, such as mean embedding kernels \citep{Smola07}, or kernels based on distances between distribution representations. Define the population operator
\[
\mathscr L(h)(\cdot)
=
\E\!\left[\kappa(P,\cdot)h(P)\right],
\]
and the proxy empirical operator
\[
\widehat{\mathscr L}_N(h)(\cdot)
=
\frac{1}{N}\sum_{i=1}^N \kappa(\widehat P_i,\cdot)h(\widehat P_i).
\]
Let \(\{(\widehat\lambda_j,\widehat\phi_j)\}_{j\ge1}\) be the eigenpairs of \(\widehat{\mathscr L}_N\), ordered by nonincreasing eigenvalues and orthonormal in \(\mathcal H_\kappa\). For \(s\ge1\), define
\[
\widehat\Phi_s(P)
=
(\widehat\phi_1(P),\ldots,\widehat\phi_s(P))^\top .
\]
Given labeled bags \(\{(\widehat P_i,Y_i)\}_{i=1}^n\), fit
\begin{equation}\label{eq:ridge_dist}
\begin{gathered}
\widehat\gamma \in \argmin_{\gamma\in\R^s}
\left\{
\frac{1}{n}\sum_{i=1}^n
\left(Y_i-\langle\gamma,\widehat\Phi_s(\widehat P_i)\rangle\right)^2
+\xi\|\gamma\|_2^2
\right\},\\
\widehat g(\cdot)=\langle\widehat\gamma,\widehat\Phi_s(\cdot)\rangle .
\end{gathered}
\end{equation}
This is the same estimator as \eqref{eq:ridge-features}, applied to distribution-valued covariates.

\subsection{Generalization bounds and rates}

For analysis, we use the same structure as in Section~\ref{sec:me_ssl}, interpreted on the covariate space \(\mathcal P\) with metric \(D\). The only new ingredient is a bag-induced proxy error bound for \(\widehat P\) relative to \(P\). The following conditions parallel those in Section~\ref{sec:me_ssl}, but we restate them in the distribution-valued notation for completeness.

\begin{itemize}[leftmargin=*]
\item \customlabel{assumption:b1}{(B1)} \emph{Sub-Gaussian noise and conditional exogeneity.}
The noise variables satisfy the conditional sub-Gaussian bound in \ref{assumption:a1}, with \(U=P\) and \(\widehat U=\widehat P\). Moreover,
\[
\E[\mu_i\mid P_i,\widehat P_i]=0 .
\]
This holds, for example, when the bag sampling mechanism is conditionally independent of \(\mu_i\) given \(P_i\).

\item \customlabel{assumption:b2}{(B2)} \emph{Bounded H\"older kernel on distributions.}
The kernel satisfies \(|\kappa(P,P)|\le 1\) and
\[
\|\kappa(\cdot,P)-\kappa(\cdot,Q)\|_{\mathcal H_\kappa}
\le
L_\kappa D(P,Q)^{\beta_\kappa},
\quad
P,Q\in\mathcal P,
\]
for constants \(L_\kappa>0\) and \(0<\beta_\kappa\le1\).

\item \customlabel{assumption:b3}{(B3)} \emph{Spectral decay and eigenfunction regularity.}
The eigenvalues \(\{\lambda_j\}_{j\ge1}\) of \(\mathscr L\) satisfy \(\lambda_j\le a^2j^{-q}\) for some \(a>0\) and \(q>2\). The corresponding eigenfunctions are uniformly bounded and H\"older continuous on \(\mathcal P\), as in Assumption~\ref{assumption:a4}.

\item \customlabel{assumption:b4}{(B4)} \emph{Bag-induced proxy error.}
There exist $C_D<\infty$ and a deterministic sequence
$\rho(m)\to0$ such that, for every $k\ge1$,
\[
\left\{
\E\!\left[
D(P,\widehat P)^k
\mid P
\right]
\right\}^{1/k}
\le
C_D\rho(m)\sqrt{k}.
\]

\item \customlabel{assumption:b5}{(B5)} \emph{Sample size and eigengap condition.}
Let \(\Delta_s=\lambda_s-\lambda_{s+1}\). Assume \(\Delta_s>0\), and let \(r_s\) be any deterministic lower bound satisfying \(0<r_s\le\Delta_s\). Assume that \(N\) and \(m\) are large enough so that, with high probability,
\[
\|\widehat{\mathscr L}_N-\mathscr L\|_{\mathrm{op}}\le \Delta_s/3 .
\]

\item \customlabel{assumption:b6}{(B6)} \emph{Target regularity.}
For some fixed constant \(R<\infty\), the best-in-class predictor satisfies
\[
\|g\|_{\mathcal H_\kappa}\le R.
\]
\end{itemize}

Assumptions~\ref{assumption:b1}--\ref{assumption:b6} are the distribution valued counterparts of the conditions in Section~\ref{sec:me_ssl}. The only additional ingredient is Assumption~\ref{assumption:b4}, which controls the bag induced proxy error and determines the perturbation level in the distribution regression problem. Specifically, write
\[
\beta'=\min\{\beta_\kappa,\beta_\phi\},
\quad
a_m=\rho(m)^{\beta'}+\rho(m)^{2\beta'} .
\]
The quantity \(a_m\) is the bag induced proxy perturbation envelope. Since \(\rho(m)\to0\), it is of order \(\rho(m)^{\beta'}\) for large \(m\), and controls both the operator perturbation and the transfer from proxy evaluation at \(\widehat P\) to latent evaluation at \(P\).

Assumption~\ref{assumption:b4} holds, for example, for regular parametric families \(\{P_\eta:\eta\in\Theta\}\) when the estimator \(\widehat\eta\) satisfies a sub-Gaussian moment bound
\[
\left\{
\E\!\left[
\|\widehat\eta-\eta\|^k\mid\eta
\right]
\right\}^{1/k}
\lesssim
m^{-1/2}\sqrt{k},
\qquad k\ge1,
\]
and the map \(\eta\mapsto P_\eta\) is Lipschitz under \(D\). In that case \(\rho(m)\asymp m^{-1/2}\). More generally, Assumption~\ref{assumption:b4} covers nonparametric distribution estimators admitting an analogous moment bound at scale \(\rho(m)\). Assumption~\ref{assumption:b5} is the distribution valued counterpart of Assumption~\ref{assumption:a5}; it holds in finite rank or spiked distributional feature models with separated leading components, and in polynomial spectrum regimes when \(N\) and \(m\) make the combined empirical and bag induced perturbation smaller than the local gap \(r_s\).

Now, define
\begin{align*}
g&\in\argmin_{h\in\mathcal H_\kappa}\E\!\left[(Y-h(P))^2\right],
\\
\epsilon^2&=
\min_{h\in\mathcal H_\kappa}\E\!\left[(f(P)-h(P))^2\right],
\end{align*}
and write \(\mathcal E_P(\widehat g)=\E[(\widehat g(P)-g(P))^2]\). As above, \(R_{s,n,N,m}^2\) denotes an upper envelope for this error.
The next proposition specializes the main excess risk bound to distribution regression, with the generic proxy perturbation term replaced by the bag induced envelope \(a_m\).

\begin{proposition}\label{prop:dist_R_bound}
Assume Assumptions~\ref{assumption:b1}--\ref{assumption:b6} and the same sampling conditions with $C_{\beta',\theta}$ replaced by $a_m$. Then
\[
\mathcal E_P(\widehat g)=O_P(R_{s,n,N,m}^2),
\]
where, for any eigengap lower bound \(0<r_s\le\Delta_s\),
\begin{equation}\label{eq:R_bound_dist}
R_{s,n,N,m}^2
:=
s^{-q}
+
\frac{1}{r_s^2}
\left\{
\frac{1}{N}+a_m
\right\}
+
\frac{s}{n}
+
\xi.
\end{equation}
\end{proposition}

The following theorem converts Proposition~\ref{prop:dist_R_bound} into regression and prediction risk bounds for the latent distribution \(P\) and its observed proxy \(\widehat P\).

\begin{theorem}\label{thm:dist_reg_pred}
Under the conditions of Proposition~\ref{prop:dist_R_bound},
\begin{align}
\E\!\left[(\widehat g(P)-f(P))^2\right]
&=
\epsilon^2+O_P(R_{s,n,N,m}^2),
\label{eq:dist_reg_err}\\
\E\!\left[(\widehat g(\widehat P)-Y)^2\right]
&=
\epsilon^2+\tau_0^2+O_P(R_{s,n,N,m}^2),
\label{eq:dist_pred_err}
\end{align}
including the proxy evaluation transfer term in the stated order.
\end{theorem}

\begin{corollary}\label{cor:dist_rate}
Assume the conditions of Proposition~\ref{prop:dist_R_bound} and \(\lambda_s\asymp s^{-q}\). If, at \(s\asymp n^{1/(q+1)}\),
\[
\frac{1}{r_s^2}
\left\{
\frac{1}{N}+a_m
\right\}
=
o\!\left(s^{-q}+\frac{s}{n}\right)
\]
and
\[
\xi
=
o\!\left(s^{-q}+\frac{s}{n}\right),
\]
then choosing \(s\asymp n^{1/(q+1)}\) yields
\[
R_{s,n,N,m}^2
=
O\!\left(n^{-q/(q+1)}\right).
\]
Consequently, the excess terms in \eqref{eq:dist_reg_err} and \eqref{eq:dist_pred_err} achieve the same rate.
\end{corollary}

Corollary~\ref{cor:dist_rate} makes explicit the three sample sizes governing distribution regression: \(n\) labeled distributions determine the regression fit, \(N\) total distributions determine the spectral representation, and \(m\) samples per distribution determine the proxy perturbation \(a_m\). The oracle feature rate is recovered only when \(N\) is large enough and \(m\) grows fast enough for the operator and proxy perturbation term to be negligible at the oracle truncation level.

\section{Experiments}\label{sec:experiments}

We evaluate the proposed estimator on two synthetic experiments and two real-world applications. The first synthetic experiment studies the noisy Euclidean covariate setting of Section~\ref{sec:me_ssl}, and the second studies the distribution regression setting of Section~\ref{sec:distribution-reg}. For each task, we report normalized root mean squared error (RMSE). On real-world datasets, we compare against feasible supervised learning baselines (SL1--SL3) and semi-supervised learning baselines (SSL). Baseline definitions and tuning are provided in Appendices~\ref{appendix:baseline} and \ref{appendix:param.selection}.

\begin{figure*}[t!]
\centering

\begin{minipage}[t]{0.63\textwidth}
\vspace{0pt}
\centering
\fontsize{8}{9}\selectfont
\resizebox{\linewidth}{!}{%
\begin{tabular}{llrrrrrrrr}
\toprule
& & \multicolumn{8}{c}{$n/N \times 100(\%)$} \\
\cmidrule{3-10}
& & 5\% & 10\% & 15\% & 20\% & 25\% & 30\% & 35\% & 40\% \\
\midrule
SL1 & SDM & 0.88 & 0.74 & 0.64 & 0.63 & 0.51 & 0.51 & 0.49 & 0.43 \\
\midrule
SL2 & $\mathrm{LASSO}_1$ & 0.90 & 0.69 & 0.65 & 0.68 & 0.66 & 0.64 & 0.63 & 0.62 \\
    & SVR                  & 1.28 & 1.24 & 1.05 & 0.98 & 0.87 & 0.85 & 0.83 & 0.82 \\
    & RF                   & 0.89 & 0.88 & 0.68 & 0.66 & 0.64 & 0.54 & 0.59 & 0.50 \\
    & KRR                  & 0.98 & 0.99 & 0.85 & 0.70 & 0.77 & 0.74 & 0.68 & 0.66 \\
\midrule
SL3 & $\mathrm{LASSO}_2$ & 2.88 & 1.69 & 2.62 & 2.78 & 2.70 & 2.39 & 2.23 & 2.10 \\
    & CNN                  & 2.52 & 1.84 & 1.44 & 1.36 & 1.45 & 1.41 & 1.34 & 1.04 \\
    & RF                   & 1.88 & 1.69 & 1.52 & 1.68 & 1.01 & 1.20 & 0.99 & 1.03 \\
\midrule
SSL & LapRLS & 1.01 & 0.79 & 0.70 & 0.63 & 0.56 & 0.54 & 0.51 & 0.48 \\
    & E-CNN  & 0.98 & 0.97 & 0.92 & 0.78 & 0.76 & 0.68 & 0.66 & 0.62 \\
\midrule
    & SSDR   & \textbf{0.75} & \textbf{0.61} & \textbf{0.58} & \textbf{0.49} & \textbf{0.45} & \textbf{0.42} & \textbf{0.41} & \textbf{0.37} \\
\bottomrule
\end{tabular}%
}

\vspace{0.3em}
\refstepcounter{table}\label{tb-galaxy}
\small
\emph{Table~\thetable.} Galaxy cluster mass prediction: normalized RMSE.
\end{minipage}
\hfill
\begin{minipage}[t]{0.35\textwidth}
\vspace{0pt}
\centering
\includegraphics[width=0.92\linewidth]{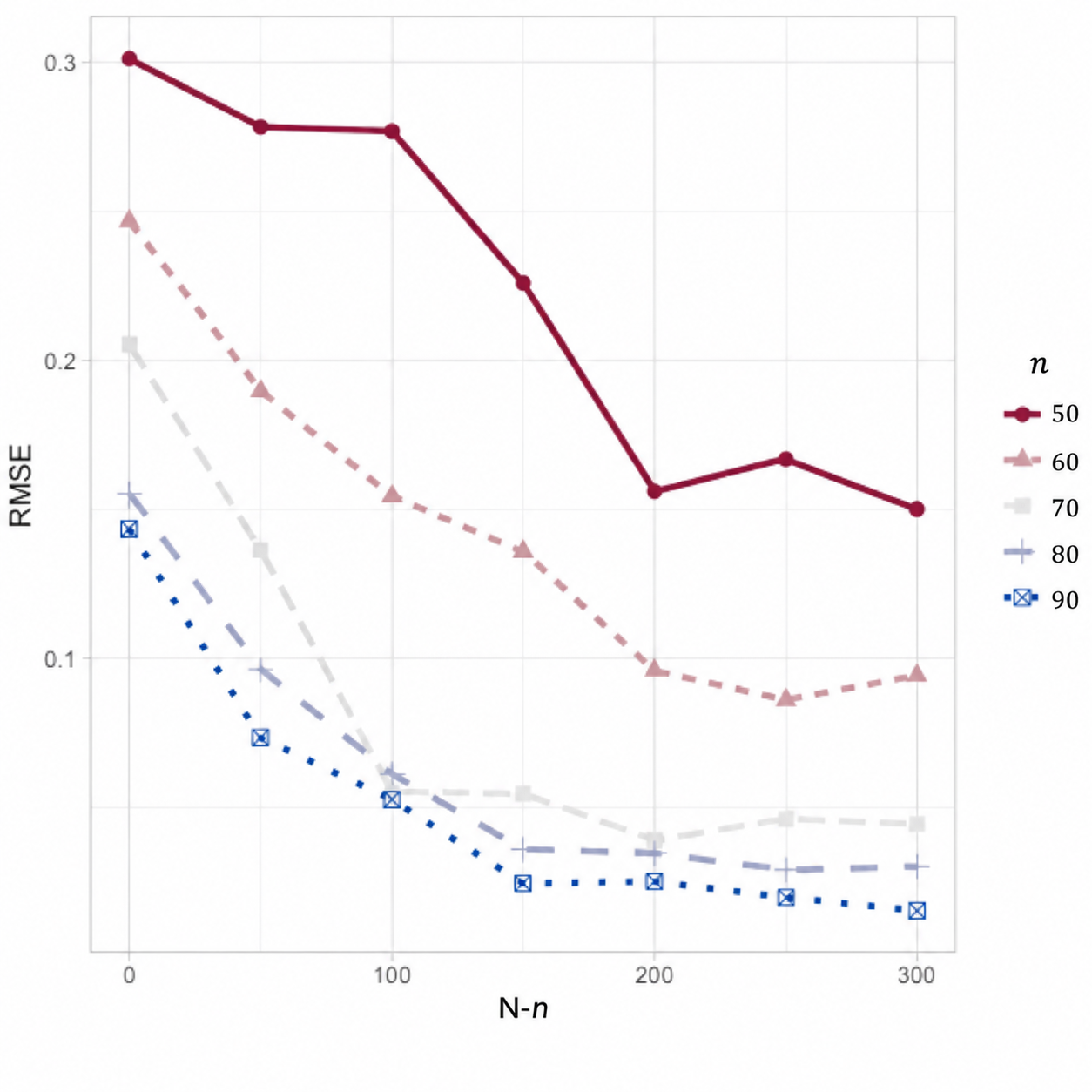}

\vspace{0.3em}
\refstepcounter{figure}\label{fig:sim-beta}
\small
\emph{Figure~\thefigure.} Normalized RMSE for predicting Beta skewness.
\end{minipage}
\end{figure*}

\subsection{Synthetic noisy Euclidean covariates}
\label{subsec:exp-noisy-euclidean}

We first evaluate the noisy covariate setting in Section~\ref{sec:me_ssl}. Latent covariates $X\in\mathbb R^{10}$ are generated from a smooth low-dimensional nonlinear structure, the response is $Y=f(X)+\epsilon$, and the learner observes only noisy proxies $\widehat X=X+\tau W$, where $W\sim N(0,I)$. We fix the number of labeled samples, vary the total number of proxy covariates $N$, and consider mild and strong proxy noise. Table~\ref{tb-noisy-euclidean} compares SSDR using all proxies, a labeled-only spectral baseline that learns features only from labeled proxies, and an oracle baseline using the true latent $X$. Increasing $N$ improves performance when proxy noise is mild; under stronger noise, gains are smaller and begin to saturate, consistent with the proxy perturbation floor in the theory. Details are provided in Appendix~\ref{append:noisy.euclidean}.

\begin{table}[h!]
\centering
\footnotesize
\setlength{\tabcolsep}{4pt}
\renewcommand{\arraystretch}{1.05}
\begin{tabular*}{\columnwidth}{@{\extracolsep{\fill}}ccccc@{}}
\toprule
$\tau$ & $N$ & SSDR & Label-only & Oracle $X$ \\
\midrule
.10 & 100  & .091 (.012) & .112 (.014) & .053 (.007) \\
.10 & 300  & .067 (.009) & .112 (.015) & .054 (.007) \\
.10 & 1000 & .057 (.008) & .113 (.014) & .053 (.008) \\
\midrule
.40 & 100  & .176 (.019) & .198 (.022) & .054 (.007) \\
.40 & 300  & .153 (.017) & .197 (.021) & .053 (.008) \\
.40 & 1000 & .142 (.016) & .199 (.022) & .054 (.007) \\
\bottomrule
\end{tabular*}
\caption{Noisy Euclidean covariates: normalized RMSE, with standard deviations in parentheses.}
\label{tb-noisy-euclidean}
\end{table}

\subsection{Synthetic distribution regression: predicting Beta skewness}
\label{subsec:exp-beta}

This simulation setup follows \citet{Poczos13}. Each covariate instance is a Beta distribution $\mathrm{Beta}(a,b)$, the response is the skewness
\[
f(a,b)=\frac{2(b-a)\sqrt{a+b+1}}{(a+b+2)\sqrt{ab}},
\]
and the covariate is observed through a finite bag of samples. Figure~\ref{fig:sim-beta} reports normalized RMSE as a function of labeled sample size $n$ and additional unlabeled sample size $N-n$, which ranges from $0$ to $300$. Error decreases as $n$ increases and also as $N-n$ increases, consistent with the role of unlabeled covariates in the bound. Details are in Appendix~\ref{append:beta.dist}.

\subsection{Galaxy cluster mass from line-of-sight velocity distributions}
\label{subsec:exp-galaxy}

We use the galaxy clusters dataset of \citet{Ntampaka15} to predict halo mass from the distribution of line-of-sight velocities $v_{\mathrm{los}}$. Prior work often relies on low-dimensional summaries, whereas SSDR uses the full bag of $v_{\mathrm{los}}$ samples as a distributional covariate. We fix the training pool at $N=1632$ clusters and vary $n$ by selecting labeled clusters from this pool. Table~\ref{tb-galaxy} shows that SSDR achieves the lowest normalized RMSE across label budgets. Details are in Appendix~\ref{appendix:galaxy}.

\subsection{Default risk from distributions of stock returns}
\label{subsec:exp-default}

We predict default probability from distributions of daily stock returns. We select the top 500 stocks by market capitalization in the Korea Stock Exchange and collect daily returns from 2010 to 2015. Labels are default probabilities obtained from FnGuide and Bloomberg. We vary the labeled fraction by randomly withholding labels while retaining all covariate bags. Table~\ref{tbl-default} shows that SSDR is competitive and typically best or tied for best, especially in the low-label regime. Details are in Appendix~\ref{appendix:default}.

\begin{table}[h!]
\centering
\footnotesize
\setlength{\tabcolsep}{3.2pt}
\renewcommand{\arraystretch}{0.98}
\begin{tabular}{@{}llrrrr@{}}
\toprule
& & \multicolumn{2}{c}{FnGuide} & \multicolumn{2}{c}{Bloomberg} \\
\cmidrule(lr){3-4}\cmidrule(lr){5-6}
& $n/N(\%)$ & 5 & 40 & 5 & 40 \\
\midrule
SL1 & SDM & .12 & .05 & .09 & .07 \\
\midrule
SL2 & $\mathrm{LASSO}_1$ & .18 & .09 & .09 & .07 \\
    & SVR                  & .15 & .09 & .10 & .09 \\
    & RF                   & .11 & .08 & .08 & .07 \\
    & KRR                  & .15 & .10 & .10 & .08 \\
\midrule
SL3 & $\mathrm{LASSO}_2$ & .51 & .39 & .18 & .16 \\
    & CNN                  & .52 & .42 & .10 & .09 \\
    & RF                   & .378 & .29 & .12 & .12 \\
\midrule
SSL & LapRLS & .07 & \textbf{.04} & .08 & .07 \\
    & E-CNN  & \textbf{.06} & .06 & .07 & .07 \\
\midrule
    & SSDR   & \textbf{.06} & .05 & \textbf{.07} & \textbf{.06} \\
\bottomrule
\end{tabular}
\caption{Default risk prediction: normalized RMSE.}
\label{tbl-default}
\end{table}

\section{Discussion}

This paper develops a unified theory for semi-supervised regression with noisy proxy covariates and shows that distribution regression is a direct specialization. The two stage estimator learns spectral features from all proxy covariates and fits ridge regression on labeled data. The generalization bounds identify when unlabeled proxies recover the oracle feature rate, with conditional validation tuning and eigensolver stability guarantees.

Our framework provides a statistical perspective of how unlabeled data and proxy measurement error jointly affect generalization in semi-supervised learning. To our knowledge, this is among the first such theoretical analyses, with potential implications for downstream prediction using foundation-model representations, where unlabeled sample size and representation quality may play distinct roles. Extending these guarantees to pretrained or fine-tuned encoders would require additional control of representation error and dependencies induced by encoder training.

The main limitation is that the fast rate requires a stable leading eigenspace and sufficiently accurate proxies. Increasing the unlabeled sample size alone does not guarantee the oracle-feature rate when proxy error remains nonvanishing. In distribution regression, this appears through the bag size dependent proxy term. For large $N\times N$ kernels, Nystr\"om or randomized eigensolvers are often required (Section~\ref{subsubsec:approx_eig}). Future work includes adaptive diagnostics for spectral and proxy quality, weaker eigengap analyses, and sharper proxy error bounds for modern learned representations.

\label{layout:main_end}


\section*{Impact Statement}
This paper studies semi-supervised regression with noisy proxy covariates. Its societal impact depends on downstream applications, with no risks identified beyond those of standard predictive modeling.

\section*{Acknowledgements}

This work was supported in part by grants from the National Research Foundation of Korea funded by the Korean government (MSIT) (Nos. RS-2022-NR068754, RS-2024-00335008, and RS-2025-24534596), by the Samsung Science and Technology Foundation under Project Number SSTF-BA2502-01, and by Korea University Grant K2612051. This work was also supported in part by the New Faculty Startup Fund from Seoul National University (No. 0670-20240073), and by Institute of Information \& communications Technology Planning \& Evaluation (IITP) grant funded by the Korea government(MSIT) (No. RS-2021-II211343, Artificial Intelligence Graduate School Program (Seoul National University)). The authors would like to thank Jaehee Seo for helpful discussions and suggestions that led to the improvement of the convergence rate analyses.

\IfFileExists{bibliography.bib}{%
\bibliography{bibliography}
\IfFileExists{icml2026.bst}{\bibliographystyle{icml2026}}{\bibliographystyle{plainnat}}%
}{%
\input{bibliography.bbl}%
}

\newpage
\appendix
\onecolumn
\begin{center}
{\large\bf APPENDIX}
\end{center}
\vspace*{.1in}

\section{Proof}
\label{sec:proof}

\subsection{Proofs for Section~\ref{sec:me_ssl}}
\label{app:proofs_me_ssl}

Throughout, $c,C>0$ are finite constants that may change from line to line. They may depend on fixed parameters $R,\tau,C_D,a,q,L_\kappa,M_\phi,c_*^{-1}$, the regularity exponents, and $\|f\|_{L_2}$, but not on $n,N,s,\xi$, a shrinking proxy scale $\theta$, or the bag size and bandwidth in the distribution setting. All conditional bias and moment constants are uniform over the stated latent class. We write $A\lesssim B$ for $A\le CB$.

\subsubsection{Notation and basic identities}

\paragraph{Asymptotic and probability conventions.}
Limits are along deterministic sequences $n\to\infty$, $s=s_n$, $N=N_n$, and proxy scales $\theta=\theta_n$; in the distribution setting $m=m_n$. The regularization $\xi$ and deterministic subspace tolerance $\varepsilon_{\rm spec}$ may also vary along this sequence. A statement ``with high probability'' means probability tending to one along this sequence. For a nonnegative random $Z_n$ and deterministic $a_n>0$,
\[
Z_n=O_P(a_n)
\quad\Longleftrightarrow\quad
\forall\delta\in(0,1)\ \exists C_\delta,n_\delta<\infty\
\forall n\ge n_\delta:\ \Pb(Z_n>C_\delta a_n)\le\delta.
\]
The quantities $R_{s,n,N}^2$ and $R_{s,n,N,m}^2$ are the deterministic expressions in \eqref{eq:R_bound} and \eqref{eq:R_bound_dist}. They are not random upper envelopes or numerical confidence bounds. Lemmas~\ref{lem:within_space} and \ref{lem:dist_within} separately provide finite-design conditional probability bounds.

\paragraph{Precise conditions for the main-text bounds.}\label{app:technical_conditions}
Theorems~\ref{prop:error-estimator-best}--\ref{thm:error-prediction} use the following explicit sufficient conditions in addition to \ref{assumption:a1}--\ref{assumption:a6}.
The latent--proxy--response triples are iid, \(f\in L_2(\Pb_U)\), and the best-in-class minimizer exists. The response noise satisfies, uniformly in \((U_i,\widehat U_i)\),
\begin{equation}\label{eq:conditional_noise_full}
\E[e^{t\mu_i}\mid U_i,\widehat U_i]\le e^{\tau^2t^2/2},\qquad t\in\R.
\end{equation}
Thus \(\E[\mu_i\mid U_i,\widehat U_i]=0\). For the chosen \(s\) and eigengap lower bound \(0<r_s\le\Delta_s\), assume
\begin{equation}\label{eq:secondstage_suff}
s^2/n\to0,\qquad n\lambda_s\ge c_*>0,\qquad
\frac{N^{-1}+C_{\beta',\theta}}{\lambda_s r_s^2}\to0.
\end{equation}
They imply the population--empirical norm comparison and residual-score control proved in Technical Lemma~\ref{tl:gram_score}, even though all covariates are reused to learn the eigenfeatures. Technical Lemma~\ref{tl:radius} then gives a deterministic radius \(B_s\) with
\[
B_s^2\lesssim\lambda_s^{-1}\lesssim r_s^{-2},\qquad
\Pb(\|\widehat g_{s,\xi}\|_{\mathcal H_\kappa}\le B_s)\to1.
\]
This is derived, rather than an additional assumed bound on the fitted function.
For Proposition~\ref{prop:dist_R_bound}, use the bag-wise iid model and the same conditional noise bound with \(U=P\), \(\widehat U=\widehat P\), together with
\begin{equation}\label{eq:dist_secondstage_suff}
s^2/n\to0,\qquad n\lambda_s\ge c_*>0,\qquad
\frac{N^{-1}+a_m}{\lambda_s r_s^2}\to0.
\end{equation}
The eigenfunction regularity in \ref{assumption:b3} is the counterpart of \ref{assumption:a4}. These sufficient conditions hold at the oracle truncation under the two-sided spectral and negligibility conditions in Corollaries~\ref{cor:error-rate} and \ref{cor:dist_rate}, as verified in their proofs. For a validation training fit, replace \(n\) by \(n_{\rm tr}\) and require the conclusions uniformly over the candidate grid.

\paragraph{Examples of the main regularity conditions.}\label{app:regularity_examples}
Concrete examples help contextualize some of these regularity conditions. Assumption~\ref{assumption:a1} covers, for example, bounded proxy perturbations or additive proxy noise with uniformly controlled moments on a compact covariate space. Assumption~\ref{assumption:a3} is standard for kernels whose integral operators have polynomial spectral decay, such as Sobolev or Mat\'ern type kernels on compact domains. 
The uniform-in-$j$ bound in Assumption~\ref{assumption:a4} is an additional condition; continuity or smoothness of each eigenfunction alone does not imply it. In the periodic Fourier case on $[0,1]$, sine and cosine eigenfunctions are uniformly bounded, while their Lipschitz constants grow with frequency. The weighted summability condition then holds if the eigenvalues decay sufficiently fast. This example verifies a concrete configuration, not a claim that every smooth Mercer kernel satisfies \ref{assumption:a4}. 
Assumption~\ref{assumption:a5} covers finite rank or spiked models with separated leading components, and polynomial spectrum regimes in which the local gap \(r_s\) is large enough relative to the operator perturbation.

\paragraph{Test expectations and fixed constants.}
Let \(\mathscr D\) denote the complete observed training sample. We write
\[
\E_*[F(X,\widehat X,Y)]
=\E[F(X,\widehat X,Y)\mid\mathscr D]
\]
when the triple on the left is a fresh independent test triple and fitted functions are held fixed. Inherited risk displays with \(\E\) have this same interpretation. Expectations used for training-sample concentration are unconditional unless conditioning is stated. Constants may depend on fixed parameters \(R,\tau,C_D,a,q,L_\kappa,M_\phi\), the regularity exponents, and \(\|f\|_{L_2}\), but not on sample sizes, \(s,\xi\), or a shrinking proxy scale.

\paragraph{Risk, norm, and operator notation.}
For a function held fixed after training, define
\[
\mathcal R_X(h)=\E_*[(Y-h(X))^2],\qquad
\mathcal R_{\widehat X}(h)=\E_*[(Y-h(\widehat X))^2].
\]
Define $\mathcal R_P$ and $\mathcal R_{\widehat P}$ by the same formulas with $P,\widehat P$ replacing $X,\widehat X$. Auxiliary numerical randomness is also fixed in $\E_*$. For $a\in\R^n$,
\[
\|a\|_n^2=n^{-1}\sum_{i=1}^na_i^2.
\]
For a bounded operator $A$ on a Hilbert space $\mathcal H$ and an orthonormal basis $(u_\ell)$,
\[
\|A\|_{\mathrm{op}}=\sup_{\|h\|_{\mathcal H}\le1}\|Ah\|_{\mathcal H},\qquad
\|A\|_{HS}^2=\sum_\ell\|Au_\ell\|_{\mathcal H}^2.
\]
The latter, when finite, does not depend on the basis. For self-adjoint operators,
\[
A\preceq B\quad\Longleftrightarrow\quad
\langle h,(B-A)h\rangle_{\mathcal H}\ge0\quad(h\in\mathcal H).
\]
The norm $L_2(\Pb_U)$ for predictors uses the latent covariate law. In density examples $L_2(dx)$ instead uses Lebesgue measure; a probability measure $P$ is identified with its density $p=dP/dx$ only when that density exists and belongs to the stated function space.

\paragraph{RKHS and \(L_2\) geometry.}
Put \(k_x=\kappa(x,\cdot)\). For \(h,h'\in\mathcal H_\kappa\), boundedness of the kernel makes \(k_Xh'(X)\) Bochner integrable, since
\[
\E\|k_Xh'(X)\|_{\mathcal H_\kappa}
\le\|h'\|_{\mathcal H_\kappa}\E\kappa(X,X)\le\|h'\|_{\mathcal H_\kappa}.
\]
Thus the reproducing property gives, step by step,
\begin{align}
\langle h,\mathscr Lh'\rangle_{\mathcal H_\kappa}
&=\left\langle h,\E[k_Xh'(X)]\right\rangle_{\mathcal H_\kappa}\nonumber\\
&=\E[h'(X)\langle h,k_X\rangle_{\mathcal H_\kappa}]
=\E[h(X)h'(X)].\label{eq:bilinear_rkhs}
\end{align}
In particular,
\begin{equation}\label{eq:L2_RKHS_identity}
\|h\|_{L_2}^2=\langle h,\mathscr Lh\rangle_{\mathcal H_\kappa},
\end{equation}
and
\begin{equation}\label{eq:L2_le_RKHS}
\|h\|_{L_2}^2\le\lambda_1\|h\|_{\mathcal H_\kappa}^2
\le a^2\|h\|_{\mathcal H_\kappa}^2.
\end{equation}
We also have \(\lambda_1\le1\), directly from
\(\E h(X)^2\le\|h\|_{\mathcal H_\kappa}^2\).

For each positive population eigenvalue put \(e_j=\sqrt{\lambda_j}\phi_j\). Then
\[
\langle e_i,e_j\rangle_{\mathcal H_\kappa}=\delta_{ij},
\qquad
\E[e_i(X)e_j(X)]=\lambda_j\delta_{ij}.
\]
If \(h=\sum_jb_je_j+h_0\), where \(h_0\in\ker\mathscr L\), then
\[
\|h\|_{\mathcal H_\kappa}^2=\sum_jb_j^2+\|h_0\|_{\mathcal H_\kappa}^2,
\qquad
\|h\|_{L_2}^2=\sum_j\lambda_jb_j^2.
\]
The kernel component is zero as an \(L_2(\Pb_X)\) class. Consequently the population RKHS projection onto
\(\mathcal F_s\) agrees with its \(L_2\) projection, where
\[
\mathcal F_s=\operatorname{span}\{e_1,\ldots,e_s\}.
\]
In both geometries,
\[
\Pi_s h=\sum_{j\le s}b_je_j.
\]
For empirical eigenfunctions, orthonormality is in \(\mathcal H_\kappa\), so
\[
\left\|\sum_{j\le s}\gamma_j\widehat\phi_j\right\|_{\mathcal H_\kappa}^2
=\|\gamma\|_2^2.
\]
This equality identifies the ridge coefficient penalty with the RKHS penalty.

\paragraph{Kernel-matrix normalization and rank.}
For unit orthogonal eigenvectors $v_j$ with $\sigma_j>0$, \eqref{eq:eig-expansion} gives
\[
\langle\widehat\phi_j,\widehat\phi_k\rangle_{\mathcal H_\kappa}
=\frac{v_j^\top\mathbf K v_k}{\sqrt{\sigma_j\sigma_k}}
=\delta_{jk},\qquad
\widehat{\mathscr L}_N\widehat\phi_j=(\sigma_j/N)\widehat\phi_j.
\]
For example, the coefficient of $k_{\widehat X_i}$ after applying $\widehat{\mathscr L}_N$ is $N^{-1}\sigma_j^{-1/2}(\mathbf Kv_j)_i=(\sigma_j/N)\sigma_j^{-1/2}(v_j)_i$. Thus $s\le\operatorname{rank}(\mathbf K)$ is required. At $\xi=0$ the minimum-norm convention makes the estimator defined even for a rank-deficient labeled design. The proofs of the main bounds use the stronger event $G_n\succ0$, derived below; this convention alone does not establish that event.

\paragraph{Targets, operators, and projectors.}
The assumed minimizer \(g\) satisfies
\[
\E[(f(X)-g(X))h(X)]=0,\qquad h\in\mathcal H_\kappa.
\]
Indeed, for every $t\in\R$,
\[
0\le\|f-g-th\|_{L_2}^2-\|f-g\|_{L_2}^2
=-2t\langle f-g,h\rangle_{L_2}+t^2\|h\|_{L_2}^2.
\]
Dividing for positive and negative $t$ and sending $t\to0$ gives the asserted orthogonality. Existence of this unconstrained minimizer is essential: the RKHS need not be closed in $L_2$. If $\overline{\mathcal H_\kappa}^{\,L_2}=L_2$ and the minimizer exists, choose $h_j\in\mathcal H_\kappa$ with $h_j\to f-g$ in $L_2$; then
\[
\|f-g\|_{L_2}^2=\lim_j\langle f-g,h_j\rangle_{L_2}=0.
\]
Thus a positive approximation error with an attained minimizer concerns non-dense hypothesis classes. A constrained-ball minimizer cannot be substituted without rechecking this orthogonality.

It follows that
\[
\E[(Y-g(X))^2]=\epsilon^2+\tau_0^2.
\]
Write
\[
\mathscr L_N=N^{-1}\sum_i k_{X_i}\otimes k_{X_i},
\qquad
\widehat{\mathscr L}_N=N^{-1}\sum_i k_{\widehat X_i}\otimes k_{\widehat X_i},
\]
where \((u\otimes v)h=u\langle v,h\rangle_{\mathcal H_\kappa}\).
Let \(\widehat\Pi_s^{\mathcal H}\) be the RKHS projector onto
\(\widehat{\mathcal F}_s=\operatorname{span}\{\widehat\phi_1,\ldots,\widehat\phi_s\}\).
Set \(g_s=\Pi_sg\), \(\widehat g_s^{\mathcal H}=\widehat\Pi_s^{\mathcal H}g\), and
\[
\Delta_s=\lambda_s-\lambda_{s+1},\quad 0<r_s\le\Delta_s,\quad
\eta_N=\|\widehat{\mathscr L}_N-\mathscr L\|_{HS},\quad
\varepsilon_N=\|\widehat{\mathscr L}_N-\mathscr L\|_{\mathrm{op}}.
\]

\paragraph{Proxy moments and the common envelope.}
Put \(u_t=\E D_2(X,\widehat X)^t\) and \(\beta'=\min\{\beta_\kappa,\beta_\phi\}\).
For \(0<t\le2\), Assumption~\ref{assumption:a1} and moment comparison give
\[
u_t\le(\E D_2(X,\widehat X)^2)^{t/2}
\le(C_D\sqrt2\,\theta)^t.
\]
Choose \(C_{\beta',\theta}\), up to one fixed multiplicative constant, so that
\[
u_{\beta_\phi}^2+u_{\beta_\kappa}^2
+u_{2\beta_\kappa}+u_{2\beta_\kappa}^{1/2}
\le C_{\beta',\theta}.
\]
The four contributions have orders
\(\theta^{2\beta_\phi},\theta^{2\beta_\kappa},
\theta^{2\beta_\kappa},\theta^{\beta_\kappa}\), respectively. In the small-proxy regime \(0<\theta\le1\), one admissible explicit choice is
\[
C_{\beta',\theta}=C\{\theta^{\beta'}+\theta^{2\beta'}\}.
\]
The square-root moment is retained because transferring squared risks, unlike bounding squared evaluation differences alone, can have a first-order proxy contribution. Write
\[
d_N^2=N^{-1}+C_{\beta',\theta}.
\]
Condition~\eqref{eq:secondstage_suff} implies
\(\eta_N/(r_s\sqrt{\lambda_s})=o_P(1)\) by Lemma~\ref{lem:op_dev}.

\subsubsection{Auxiliary lemmas}

\begin{lemma}[Bounded proxy spectral features]\label{lem:feat_bound}
Under Assumption~\ref{assumption:a2}, for every $s\ge 1$,
\[
\|\widehat\Phi_s(\widehat X)\|_2\le 1
\]
almost surely.
\end{lemma}

\begin{proof}
Because $\{\widehat\phi_j\}_{j\ge 1}$ is orthonormal in $\mathcal H_\kappa$,
\[
\|\widehat\Phi_s(\widehat X)\|_2^2
=
\sum_{j=1}^s \widehat\phi_j(\widehat X)^2
=
\sum_{j=1}^s
\langle \widehat\phi_j,\kappa(\widehat X,\cdot)\rangle_{\mathcal H_\kappa}^2
\le
\|\kappa(\widehat X,\cdot)\|_{\mathcal H_\kappa}^2
=
\kappa(\widehat X,\widehat X)
\le 1,
\]
where the inequality follows from Bessel's inequality.
\end{proof}

\begin{lemma}[Approximation by top-$s$ eigenfunctions]\label{lem:approx_rate}
Assume Assumptions~\ref{assumption:a3} and \ref{assumption:a6}. Let $\Pi_s$ be the $L_2(\Pb_X)$-orthogonal projection onto
$\mathrm{span}\{\phi_1,\ldots,\phi_s\}$ and set $g_s=\Pi_s g$. Then
\[
\|g-g_s\|_{L_2(\Pb_X)}^2
\le
\lambda_{s+1}\|g\|_{\mathcal H_\kappa}^2
\lesssim
s^{-q}.
\]
\end{lemma}

\begin{proof}
Use the positive-spectrum RKHS basis \(e_j=\sqrt{\lambda_j}\phi_j\) from A.1.1 and write
\[
g=\sum_{j\ge1}b_je_j+g_0,\qquad g_0\in\ker\mathscr L,\qquad
\sum_jb_j^2\le\|g\|_{\mathcal H_\kappa}^2\le R^2.
\]
The possible kernel component has zero population \(L_2\) norm.
Since \(g_s=\sum_{j\le s}b_je_j\), orthonormality of the \(\phi_j\)'s gives
\[
\|g-g_s\|_{L_2}^2
=\sum_{j>s}\lambda_jb_j^2
\le\lambda_{s+1}\sum_{j>s}b_j^2
\le R^2\lambda_{s+1}
\le a^2R^2(s+1)^{-q}.
\]
This also explains the normalization factors in the statement.
\end{proof}

\begin{lemma}
\label{lem:x-hatx-distance}
Under Assumption~\ref{assumption:a1}, for every
$0<\beta\le1$,
\[
\frac1N
\sum_{i=1}^N
D_2(\widehat X_i,X_i)^\beta
=
\mathcal O_P(\theta^\beta).
\]
\end{lemma}

\begin{proof}
By conditional moment comparison and then averaging over \(X\),
\[
\E D_2(X,\widehat X)^\beta
\le\{\E D_2(X,\widehat X)\}^{\beta}
\le(C_D\theta)^\beta,\qquad 0<\beta\le1.
\]
For \(A_N=N^{-1}\sum_iD_2(X_i,\widehat X_i)^\beta\), linearity gives
\(\E A_N\le(C_D\theta)^\beta\).
Thus, for every \(M>0\),
\[
\Pb(A_N>M\theta^\beta)\le C_D^\beta/M.
\]
Given any fixed error tolerance, choose \(M\) large enough independently of \(N\); this is exactly \(A_N=O_P(\theta^\beta)\).
For the powers \(0<\beta\le2\) needed elsewhere, the same argument uses
\(\E D^\beta\le(\E D^2)^{\beta/2}\) and changes only the fixed constant.
\end{proof}

\begin{lemma}[Proxy operator deviation]\label{lem:op_dev}
Assume Assumptions~\ref{assumption:a1}--\ref{assumption:a4}. Let
\[
\mathscr L_N(h)
=
\frac{1}{N}\sum_{i=1}^N \kappa(X_i,\cdot)h(X_i),
\qquad
\widehat{\mathscr L}_N(h)
=
\frac{1}{N}\sum_{i=1}^N \kappa(\widehat X_i,\cdot)h(\widehat X_i).
\]
Then
\[
\|\widehat{\mathscr L}_N-\mathscr L\|_{HS}
=
\mathcal O_P\!\left(N^{-1/2}+C_{\beta',\theta}^{1/2}\right),
\qquad
\|\widehat{\mathscr L}_N-\mathscr L\|_{HS}^2
=
\mathcal O_P\!\left(\frac{1}{N}+C_{\beta',\theta}\right).
\]
\end{lemma}

\begin{proof}
For any \(u,v\in\mathcal H_\kappa\), an arbitrary orthonormal basis \(\{a_\ell\}\) gives
\[
\|u\otimes v\|_{HS}^2
=\sum_\ell\|u\langle v,a_\ell\rangle\|_{\mathcal H_\kappa}^2
=\|u\|_{\mathcal H_\kappa}^2\|v\|_{\mathcal H_\kappa}^2.
\]
In particular, with \(T_i=k_{X_i}\otimes k_{X_i}\),
\(\|T_i\|_{HS}=\kappa(X_i,X_i)\le1\) and \(\E T_i=\mathscr L\).
First decompose
\[
\widehat{\mathscr L}_N-\mathscr L
=(\mathscr L_N-\mathscr L)+(\widehat{\mathscr L}_N-\mathscr L_N).
\]

\emph{Step 1: sampling variation.}
Independence and centering eliminate the cross terms:
\begin{align*}
\E\|\mathscr L_N-\mathscr L\|_{HS}^2
&=N^{-2}\sum_{i,j}
\E\langle T_i-\mathscr L,T_j-\mathscr L\rangle_{HS}\\
&=N^{-2}\sum_i\E\|T_i-\mathscr L\|_{HS}^2\\
&\le N^{-1}.
\end{align*}
Therefore
\[
\Pb(\|\mathscr L_N-\mathscr L\|_{HS}>M/\sqrt N)\le M^{-2},
\]
which proves \(O_P(N^{-1/2})\) without invoking an unspecified concentration theorem.

\emph{Step 2: a single proxy perturbation.}
Let \(\widehat T_i=k_{\widehat X_i}\otimes k_{\widehat X_i}\) and
\(\delta k_i=k_{\widehat X_i}-k_{X_i}\). The exact identity
\[
\widehat T_i-T_i
=\delta k_i\otimes k_{\widehat X_i}+k_{X_i}\otimes\delta k_i
\]
and the rank-one norm formula yield
\begin{align*}
\|\widehat T_i-T_i\|_{HS}
&\le(\|k_{\widehat X_i}\|_{\mathcal H_\kappa}
       +\|k_{X_i}\|_{\mathcal H_\kappa})\|\delta k_i\|_{\mathcal H_\kappa}\\
&\le2L_\kappa D_2(X_i,\widehat X_i)^{\beta_\kappa}.
\end{align*}
This argument works on all of \(\mathcal H_\kappa\), including a possible null space of the population operator.

\emph{Step 3: average and take moments.}
By the triangle inequality,
\[
\|\widehat{\mathscr L}_N-\mathscr L_N\|_{HS}
\le\frac{2L_\kappa}{N}\sum_iD_2(X_i,\widehat X_i)^{\beta_\kappa}.
\]
The expectation of the right side is \(2L_\kappa u_{\beta_\kappa}\).
Markov's inequality gives
\[
\|\widehat{\mathscr L}_N-\mathscr L_N\|_{HS}
=O_P(u_{\beta_\kappa})
=O_P(C_{\beta',\theta}^{1/2}).
\]
If that moment is zero, the perturbation is zero almost surely. Combining the two terms proves the norm bound; applying
\((a+b)^2\le2a^2+2b^2\) proves its squared version. The spectral and eigenfunction assumptions in the statement remain available for the regression analysis; boundedness and kernel regularity already suffice for this operator calculation.
\end{proof}

Recall the definition of $\mathcal{H}_{\kappa}$-orthogonal
projections $\Pi_{s}$ and $\widehat{\Pi}_{s}^{\mathcal{H}}$, the  eigengap $\Delta_s$, and its lower bound $r_s$. 
The next lemma converts the operator deviation bound into a perturbation bound for the leading spectral subspace learned from the proxy covariates.

\begin{lemma}[Eigenspace perturbation]
\label{lem:eigenspace-bound}
Let
\[
\Delta_s=\lambda_s-\lambda_{s+1}>0,
\]
and let $\Pi_s$ and $\widehat\Pi_s^{\mathcal H}$ denote the
$\mathcal H_\kappa$-orthogonal projectors onto the leading
$s$-dimensional eigenspaces of $\mathscr L$ and
$\widehat{\mathscr L}_N$, respectively. On the gap event
$\varepsilon_N\le\Delta_s/3$ from Assumption~\ref{assumption:a5},
\[
\left\|
\widehat\Pi_s^{\mathcal H}-\Pi_s
\right\|_{\mathcal H_\kappa,\mathrm{op}}
\le
\frac{
2\|\widehat{\mathscr L}_N-\mathscr L\|_{\mathrm{op}}
}{
\Delta_s-\|\widehat{\mathscr L}_N-\mathscr L\|_{\mathrm{op}}
}
\le
\frac{
3\|\widehat{\mathscr L}_N-\mathscr L\|_{\mathrm{op}}
}{
\Delta_s
}
\le
\frac{
3\|\widehat{\mathscr L}_N-\mathscr L\|_{HS}
}{
r_s
}.
\]
\end{lemma}

\begin{proof}
Write
\[
E_N=\widehat{\mathscr L}_N-\mathscr L,
\qquad
\varepsilon_N=\|E_N\|_{\mathrm{op}}.
\]
We work on the stated gap event, so $\varepsilon_N\le\Delta_s/3$. 
Let
\[
\widehat\lambda_1\ge\widehat\lambda_2\ge\cdots
\]
denote the eigenvalues of $\widehat{\mathscr L}_N$.
Since $\mathscr L$ and $\widehat{\mathscr L}_N$ are compact
self-adjoint operators, the min--max principle gives
\[
|\widehat\lambda_j-\lambda_j|
\le
\varepsilon_N,
\qquad j\ge1.
\]

For clarity, the min--max comparison applies because for every \(v\),
\[
\langle v,\mathscr Lv\rangle-\varepsilon_N\|v\|^2
\le\langle v,\widehat{\mathscr L}_Nv\rangle
\le\langle v,\mathscr Lv\rangle+\varepsilon_N\|v\|^2.
\]
Taking the same infima and suprema in the Rayleigh quotient on the three expressions proves the displayed eigenvalue bound.

Consequently,
\[
\widehat\lambda_s-\lambda_{s+1}
\ge
\Delta_s-\varepsilon_N>0,
\qquad
\lambda_s-\widehat\lambda_{s+1}
\ge
\Delta_s-\varepsilon_N>0.
\]
Thus the leading $s$-dimensional empirical and population
spectral subspaces remain separated.

Set
\[
P=\Pi_s,
\qquad
\widehat P=\widehat\Pi_s^{\mathcal H}.
\]
Because $P$ and $\widehat P$ are spectral projectors,
$P$ commutes with $\mathscr L$ and $\widehat P$ commutes
with $\widehat{\mathscr L}_N$.

First consider
\[
X=(I-P)\widehat P:
\operatorname{Ran}(\widehat P)\longrightarrow
\operatorname{Ran}(I-P).
\]
Let
\[
\mathscr L_-=
\mathscr L|_{\operatorname{Ran}(I-P)},
\qquad
\widehat{\mathscr L}_+
=
\widehat{\mathscr L}_N|_{\operatorname{Ran}(\widehat P)}.
\]
Using
$\widehat{\mathscr L}_N=\mathscr L+E_N$
and the commutation relations,
\[
\mathscr L_-X-X\widehat{\mathscr L}_+
=
-(I-P)E_N\widehat P.
\]
Moreover,
\[
\mathscr L_-\preceq\lambda_{s+1}I,
\qquad
\widehat{\mathscr L}_+
\succeq\widehat\lambda_s I.
\]
For completeness, if self-adjoint operators $A\preceq aI$
and $B\succeq bI$ satisfy $b>a$, then the Sylvester equation
\[
AZ-ZB=Y
\]
has
\[
\|Z\|_{\mathrm{op}}
\le
\frac{\|Y\|_{\mathrm{op}}}{b-a}.
\]
For the actual solution $Z$, we differentiate to obtain
\[
\frac{d}{dt}(e^{tA}Ze^{-tB})=e^{tA}(AZ-ZB)e^{-tB}=e^{tA}Ye^{-tB}.
\]
The spectral bounds imply
\[
\|e^{tA}Ze^{-tB}\|_{\rm op}\le e^{-t(b-a)}\|Z\|_{\rm op}\to0.
\]
Integrating from $0$ to $T$ and letting $T\to\infty$ yields
\[
Z=-\int_0^\infty e^{tA}Ye^{-tB}\,dt.
\]
This also proves uniqueness by taking $Y=0$. The integral is norm convergent, since its integrand has norm at most $e^{-t(b-a)}\|Y\|_{\rm op}$. Consequently
\[
\|Z\|_{\mathrm{op}}
\le
\int_0^\infty
e^{-t(b-a)}\,dt\,\|Y\|_{\mathrm{op}}
=
\frac{\|Y\|_{\mathrm{op}}}{b-a}.
\]
Applying this bound gives
\[
\|(I-P)\widehat P\|_{\mathrm{op}}
\le
\frac{\varepsilon_N}
{\widehat\lambda_s-\lambda_{s+1}}
\le
\frac{\varepsilon_N}
{\Delta_s-\varepsilon_N}.
\]

Similarly, consider
\[
\widetilde X=(I-\widehat P)P.
\]
Define the restricted operators
\[
\widehat{\mathscr L}_{-}
:=
\left.\widehat{\mathscr L}_N
\right|_{\operatorname{Ran}(I-\widehat P)},
\qquad
\mathscr L_{+}
:=
\left.\mathscr L
\right|_{\operatorname{Ran}(P)}.
\]
Commutation gives
\[
\widehat{\mathscr L}_{-}\widetilde X-\widetilde X\mathscr L_{+}
=(I-\widehat P)E_NP,\qquad
\widehat{\mathscr L}_{-}\preceq\widehat\lambda_{s+1}I,\quad
\mathscr L_{+}\succeq\lambda_sI.
\]
The same integral bound with gap $\lambda_s-\widehat\lambda_{s+1}>0$ gives
\[
\|(I-\widehat P)P\|_{\mathrm{op}}
\le
\frac{\varepsilon_N}
{\lambda_s-\widehat\lambda_{s+1}}
\le
\frac{\varepsilon_N}
{\Delta_s-\varepsilon_N}.
\]
Since
\[
P-\widehat P
=
P(I-\widehat P)-(I-P)\widehat P,
\]
the triangle inequality yields
\[
\|P-\widehat P\|_{\mathrm{op}}
\le
\frac{2\varepsilon_N}
{\Delta_s-\varepsilon_N}.
\]
Finally, $\varepsilon_N\le\Delta_s/3$ implies
\[
\frac{2}{\Delta_s-\varepsilon_N}
\le
\frac{3}{\Delta_s},
\]
and hence
\[
\|P-\widehat P\|_{\mathrm{op}}
\le
\frac{3\varepsilon_N}{\Delta_s}
\le
\frac{
3\|\widehat{\mathscr L}_N-\mathscr L\|_{HS}
}{
r_s
},
\]
where the last inequality uses
$\|\cdot\|_{\mathrm{op}}\le\|\cdot\|_{HS}$
and $r_s\le\Delta_s$.
\end{proof}

\begin{technicallemma}[Alignment and population covariance]\label{tl:alignment}
On the gap event \(\varepsilon_N\le\Delta_s/3\), define the RKHS isometries
\[
U_s a=\sum_{j\le s}a_je_j,\qquad
V_s a=\sum_{j\le s}a_j\widehat\phi_j,\qquad
\Lambda_s=\operatorname{diag}(\lambda_1,\ldots,\lambda_s).
\]
There is a measurable orthogonal matrix \(O_s\) such that, with \(A_s=V_s-U_sO_s\),
\[
\|A_s\|_{\mathrm{op}}\le\|A_s\|_{HS}
\le C\eta_N/r_s.
\]
Furthermore, for
\(\Sigma_s=V_s^*\mathscr LV_s\) and
\(\widehat D_s=\operatorname{diag}(\widehat\lambda_1,\ldots,\widehat\lambda_s)\),
\[
\|\Sigma_s-\widehat D_s\|_{\mathrm{op}}\le\varepsilon_N,
\qquad
\lambda_{\min}(\Sigma_s)\ge\lambda_s-2\varepsilon_N\ge\lambda_s/3.
\]
\end{technicallemma}

\begin{proof}
\emph{Step 1: strengthen the projector calculation to Hilbert--Schmidt norm.}
In Lemma~\ref{lem:eigenspace-bound}, each Sylvester integral also satisfies
\[
\left\|\int_0^\infty e^{tA}Ye^{-tB}\,dt\right\|_{HS}
\le\int_0^\infty e^{-t(b-a)}\|Y\|_{HS}\,dt
=\frac{\|Y\|_{HS}}{b-a}.
\]
Here the right-hand sides are compressions of the Hilbert--Schmidt operator
\(E_N=\widehat{\mathscr L}_N-\mathscr L\), so their Hilbert--Schmidt norms are at most \(\eta_N\).
For \(P=U_sU_s^*\), \(\widehat P=V_sV_s^*\), the two off-diagonal bounds and the projector identity give
\[
\|P-\widehat P\|_{HS}^2
=\|(I-P)\widehat P\|_{HS}^2+\|(I-\widehat P)P\|_{HS}^2
\le\frac{2\eta_N^2}{(\Delta_s-\varepsilon_N)^2}.
\]

\emph{Step 2: align bases, rather than individual eigenvectors.}
Take a singular value decomposition \(U_s^*V_s=L\,\operatorname{diag}(c_j)R^\top\),
where \(0\le c_j\le1\), and put \(O_s=LR^\top\). Finite-dimensional trace identities show
\[
\|V_s-U_sO_s\|_{HS}^2=2\sum_j(1-c_j)
\le2\sum_j(1-c_j^2)=\|P-\widehat P\|_{HS}^2.
\]
For clarity, the projector trace identity used here is
\[
\|P-\widehat P\|_{HS}^2
=\operatorname{tr}(P)+\operatorname{tr}(\widehat P)-2\operatorname{tr}(P\widehat P)
=2s-2\|U_s^*V_s\|_F^2
=2\sum_j(1-c_j^2).
\]
The two off-diagonal squared HS norms are each $s-\operatorname{tr}(P\widehat P)$, proving their sum identity as well. Hence 
\begin{align*}
\|A_s\|_{HS}^2
&\le \|P-\widehat P\|_{HS}^2
\le \frac{2\eta_N^2}{(\Delta_s-\varepsilon_N)^2}\\
&\le \frac{9\eta_N^2}{2\Delta_s^2}
\le \frac{9\eta_N^2}{2r_s^2},
\end{align*}
where the penultimate inequality uses
$\varepsilon_N\le\Delta_s/3$, and the last inequality uses
$0<r_s\le\Delta_s$. Taking square roots gives
\[
\|A_s\|_{HS}
\le \frac{3}{\sqrt{2}}\frac{\eta_N}{r_s}
\le C\frac{\eta_N}{r_s}.
\]
A fixed measurable choice of singular vectors, or of a polar-factor extension on a singular set, suffices; the bounds do not depend on that choice.
No simplicity of the eigenvalues inside the top block is required.

\emph{Step 3: compress the operator difference.}
The empirical eigen-equations and \(V_s^*V_s=I_s\) imply
\[
\Sigma_s-\widehat D_s
=V_s^*(\mathscr L-\widehat{\mathscr L}_N)V_s.
\]
Its norm is at most \(\varepsilon_N\). The eigenvalue comparison in Lemma~\ref{lem:eigenspace-bound} then gives
\(\lambda_{\min}(\Sigma_s)\ge\widehat\lambda_s-\varepsilon_N
\ge\lambda_s-2\varepsilon_N\).
\end{proof}

\begin{technicallemma}[The two projections onto the learned space]\label{tl:two_projections}
Assume the standing kernel and target conditions of Theorem~\ref{prop:error-estimator-best}, in particular $\|g\|_{\mathcal H_\kappa}\le R$. Let \(h_H=\widehat\Pi_s^{\mathcal H}g\) and let \(h_2\) 
be the \(L_2(\Pb_X)\) projection of \(g\) onto \(\widehat{\mathcal F}_s\).
On the gap event,
\[
\|h_2-h_H\|_{L_2}^2
\le\frac{R^2\varepsilon_N^2}{\lambda_s-2\varepsilon_N},
\qquad
\|h_2-h_H\|_{\mathcal H_\kappa}
\le\frac{R\varepsilon_N}{\lambda_s-2\varepsilon_N}.
\]
In particular, \(\|h_2\|_{\mathcal H_\kappa}\le2R\).
\end{technicallemma}

\begin{proof}
Technical Lemma~\ref{tl:alignment} makes the \(L_2\) Gram matrix on the finite-dimensional space positive definite, so \(h_2\) exists uniquely there.
Write \(v=h_2-h_H\). The two distinct orthogonality relations are
\[
\langle\mathscr Lv,h_2-g\rangle_{\mathcal H_\kappa}=0,
\qquad
\langle\widehat{\mathscr L}_Nv,h_H-g\rangle_{\mathcal H_\kappa}=0,
\]
because \(\widehat{\mathscr L}_Nv\in\widehat{\mathcal F}_s\).
Subtracting them yields
\[
\|v\|_{L_2}^2
=\langle(\mathscr L-\widehat{\mathscr L}_N)v,g-h_H\rangle_{\mathcal H_\kappa}
\le R\varepsilon_N\|v\|_{\mathcal H_\kappa}.
\]
The covariance lower bound gives
\(\|v\|_{L_2}^2\ge(\lambda_s-2\varepsilon_N)\|v\|_{\mathcal H_\kappa}^2\).
If \(v=0\) both claims are immediate; otherwise divide the last two bounds by the positive norms to obtain the displayed estimates.
Finally \(\varepsilon_N\le\lambda_s/3\), so
\(\|h_2-h_H\|_{\mathcal H_\kappa}\le R\) and
\(\|h_H\|_{\mathcal H_\kappa}\le R\).
This lemma records the distinction between the two geometries; the regression proof below may instead use a norm-controlled transported comparator.
\end{proof}

\begin{lemma}[Subspace perturbation]
\label{lem:subspace_term}
Assume Assumptions~\ref{assumption:a1}--\ref{assumption:a6}.
Let
\[
g_s=\Pi_s g,
\qquad
\widehat g_s^{\mathcal H}
=
\widehat\Pi_s^{\mathcal H}g.
\]
Then
\[
\|\widehat g_s^{\mathcal H}-g_s\|_{L_2(\Pb_X)}^2
=
\mathcal O_P\!\left\{
\frac{1}{r_s^2}
\left(
\frac{1}{N}
+
C_{\beta',\theta}
\right)
\right\}.
\]
\end{lemma}

\begin{proof}
Since
\[
\widehat g_s^{\mathcal H}-g_s
=
(\widehat\Pi_s^{\mathcal H}-\Pi_s)g,
\]
the identity \eqref{eq:L2_le_RKHS} and
Assumption~\ref{assumption:a6} give
\begin{align*}
\|\widehat g_s^{\mathcal H}-g_s\|_{L_2(\Pb_X)}^2
&\le
a^2
\|\widehat g_s^{\mathcal H}-g_s\|_{\mathcal H_\kappa}^2
\\
&\le
a^2
\left\|
\widehat\Pi_s^{\mathcal H}-\Pi_s
\right\|_{\mathcal H_\kappa,\mathrm{op}}^2
\|g\|_{\mathcal H_\kappa}^2
\\
&\le
a^2R^2
\left\|
\widehat\Pi_s^{\mathcal H}-\Pi_s
\right\|_{\mathcal H_\kappa,\mathrm{op}}^2.
\end{align*}
By Lemma~\ref{lem:eigenspace-bound},
\[
\left\|
\widehat\Pi_s^{\mathcal H}-\Pi_s
\right\|_{\mathcal H_\kappa,\mathrm{op}}
\le
\frac{
3\|\widehat{\mathscr L}_N-\mathscr L\|_{HS}
}{
r_s
},
\]
and hence
\[
\|\widehat g_s^{\mathcal H}-g_s\|_{L_2(\Pb_X)}^2
\le
\frac{
9a^2R^2
}{
r_s^2
}
\|\widehat{\mathscr L}_N-\mathscr L\|_{HS}^2.
\]
Finally, Lemma~\ref{lem:op_dev} gives
\[
\|\widehat{\mathscr L}_N-\mathscr L\|_{HS}^2
=
\mathcal O_P\!\left(
\frac{1}{N}
+
C_{\beta',\theta}
\right),
\]
so
\[
\|\widehat g_s^{\mathcal H}-g_s\|_{L_2(\Pb_X)}^2
=
\mathcal O_P\!\left\{
\frac{1}{r_s^2}
\left(
\frac{1}{N}
+
C_{\beta',\theta}
\right)
\right\}.
\]
\end{proof}

\begin{lemma}[Proxy evaluation transfer]\label{lem:proxy_transfer}
Assume Assumptions~\ref{assumption:a1} and \ref{assumption:a2}. Let
\[
\mathcal G_B
=
\{h\in\mathcal H_\kappa:\|h\|_{\mathcal H_\kappa}\le B\}.
\]
Then, uniformly over $h\in\mathcal G_B$,
\[
\E\!\left[(h(\widehat X)-h(X))^2\right]
\lesssim
B^2C_{\beta',\theta}.
\]
Moreover,
\[
\left|
\E\!\left[(Y-h(\widehat X))^2\right]
-
\E\!\left[(Y-h(X))^2\right]
\right|
\lesssim
(1+B^2)C_{\beta',\theta}.
\]
\end{lemma}

\begin{proof}
\emph{Step 1: pointwise reproduction and squared evaluation error.}
For every \(h\in\mathcal G_B\),
\[
|h(\widehat X)-h(X)|
=|\langle h,k_{\widehat X}-k_X\rangle_{\mathcal H_\kappa}|
\le BL_\kappa D_2(X,\widehat X)^{\beta_\kappa}.
\]
Squaring and taking the test expectation yields
\[
\E_*[(h(\widehat X)-h(X))^2]
\le B^2L_\kappa^2u_{2\beta_\kappa}
\lesssim B^2C_{\beta',\theta}.
\]

\emph{Step 2: the two Cauchy--Schwarz factors.}
The exact factorization is
\[
(Y-h(\widehat X))^2-(Y-h(X))^2
=\{h(X)-h(\widehat X)\}\{2Y-h(\widehat X)-h(X)\}.
\]
The first factor has \(L_2\) norm at most
\(BL_\kappa u_{2\beta_\kappa}^{1/2}\).
For the second, boundedness of the kernel gives \(|h(x)|\le B\), and hence
\[
\|2Y-h(\widehat X)-h(X)\|_{L_2}
\le2\|Y\|_{L_2}+2B.
\]
Thus
\begin{align*}
|\mathcal R_{\widehat X}(h)-\mathcal R_X(h)|
&\le2BL_\kappa(\|Y\|_{L_2}+B)u_{2\beta_\kappa}^{1/2}\\
&\le C(1+B^2)C_{\beta',\theta}.
\end{align*}
All bounds are uniform over the RKHS ball. Therefore a data-dependent \(h\) may be substituted for each fixed training realization on the event \(\|h\|_{\mathcal H_\kappa}\le B\), using a fresh test triple. No independence between that \(h\) and its training covariates is needed.
\end{proof}

\begin{technicallemma}[Gram comparison and residual scores with reused covariates]\label{tl:gram_score}
Assume the standing assumptions of Theorem~\ref{prop:error-estimator-best} and \eqref{eq:secondstage_suff}. Use 
\(U_s,V_s,O_s,A_s,\Lambda_s\) from Technical Lemma~\ref{tl:alignment}, and set
\[
\varphi_s(x)=(\phi_1(x),\ldots,\phi_s(x))^\top,\qquad
w_s(x)=\Lambda_s^{-1/2}O_s\widehat\Phi_s(x).
\]
Let \(Z\) have rows \(\widehat\Phi_s(\widehat X_i)^\top\), \(W\) have rows \(w_s(\widehat X_i)^\top\), and put
\[
G_n=Z^\top Z/n,\qquad
\Sigma_s=\E_*[\widehat\Phi_s(X)\widehat\Phi_s(X)^\top].
\]
With probability tending to one, \(G_n\) is positive definite and
\begin{equation}\label{eq:gram_comparison}
\Sigma_s\preceq4G_n,\qquad G_n\succeq(\lambda_s/2)I_s.
\end{equation}
Write \(g_s=U_sb_s\), \(\|b_s\|_2\le R\), and define the transported comparator
\[
h_s^0=V_sO_s^\top b_s.
\]
Then \(\|h_s^0\|_{\mathcal H_\kappa}\le R\), and, on the gap event,
\begin{equation}\label{eq:transport_approx}
\|h_s^0-g\|_{L_2}^2
\le2R^2\lambda_{s+1}+C R^2\eta_N^2/r_s^2.
\end{equation}
For \(r_i=f(X_i)-h_s^0(\widehat X_i)\), on the event in \eqref{eq:gram_comparison},
\begin{equation}\label{eq:full_residual_score}
\left\|G_n^{-1/2}\frac{Z^\top r}{n}\right\|_2^2
=O_P\!\left(\frac{s}{n}+\frac{d_N^2}{r_s^2}\right).
\end{equation}
Every summand treated as independent in the proof below is a function of the deterministic population basis; the learned basis itself is never treated as an iid sample.
\end{technicallemma}

\begin{proof}
\emph{Step 1: the population-feature Gram matrix.}
Let \(W_0\) have rows \(\varphi_s(X_i)^\top\).
The rows are iid, \(\E[\varphi_s(X)\varphi_s(X)^\top]=I_s\), and
\(\|\varphi_s(X)\|_2^2\le M_\phi^2s\). Hilbert-space variance calculation in the finite-dimensional matrix space gives
\[
\E\|W_0^\top W_0/n-I_s\|_F^2
\le M_\phi^4s^2/n=o(1).
\]
Thus \(\alpha_n=\|W_0^\top W_0/n-I_s\|_{\mathrm{op}}=o_P(1)\).

\emph{Step 2: compare learned features to population features.}
The identity \(O_sV_s^*=U_s^*+O_sA_s^*\) gives
\begin{equation}
w_s(\widehat X_i)-\varphi_s(X_i)
=\Lambda_s^{-1/2}O_sA_s^*k_{X_i}
+\Lambda_s^{-1/2}O_sV_s^*(k_{\widehat X_i}-k_{X_i}).\label{eq:feature_perturb_split}
\end{equation}
Put \(E=W-W_0\), \(\bar D_n=n^{-1}\sum_iD_2(X_i,\widehat X_i)^{2\beta_\kappa}\).
Reproduction, isometry of \(V_s\), and Technical Lemma~\ref{tl:alignment} imply
\begin{equation}\tag{T3a}\label{eq:normalized_training_discrepancy}
\epsilon_n^2:=n^{-1}\|E\|_F^2
\le\frac{C}{\lambda_s}\left(\frac{\eta_N^2}{r_s^2}+\bar D_n\right).
\end{equation}
For all exponents needed here, the moment argument in A.1.1 and Markov's inequality give
\(\bar D_n=O_P(u_{2\beta_\kappa})\).
By \(u_{2\beta_\kappa}\le C_{\beta',\theta}\le d_N^2\) and \eqref{eq:secondstage_suff},
\(\epsilon_n=o_P(1)\). Now expand the Gram matrix, using operator rather than Frobenius norm for its cross terms:
\begin{align*}
\|W^\top W/n-I_s\|_{\mathrm{op}}
&\le\alpha_n
  +2\|W_0/\sqrt n\|_{\mathrm{op}}\|E/\sqrt n\|_{\mathrm{op}}
  +\|E/\sqrt n\|_{\mathrm{op}}^2\\
&\le\alpha_n+2\sqrt{1+\alpha_n}\,\epsilon_n+\epsilon_n^2=o_P(1).
\end{align*}
This avoids an unnecessary factor \(\sqrt s\).

\emph{Step 3: the fresh-sample norm.}
For a fresh latent \(X\), only the first term of \eqref{eq:feature_perturb_split} is present. Hence
\[
\epsilon_*^2
:=\E_*\|w_s(X)-\varphi_s(X)\|_2^2
\le C\eta_N^2/(\lambda_sr_s^2)=o_P(1).
\]
For unit vectors \(a,b\), Cauchy--Schwarz gives
\[
|\E_*[(a^\top\varphi_s(X))
       \{w_s(X)-\varphi_s(X)\}^\top b]|
\le\epsilon_*.
\]
Consequently
\[
\|\E_*[w_s(X)w_s(X)^\top]-I_s\|_{\mathrm{op}}
\le2\epsilon_*+\epsilon_*^2=o_P(1).
\]
On an event whose probability tends to one, both the empirical and test covariance matrices in the \(w_s\) coordinates are between \(\tfrac12I_s\) and \(2I_s\).
Since \(Z=W\Lambda_s^{1/2}O_s\), the two claims in \eqref{eq:gram_comparison} follow by congruence.

\emph{Step 4: a comparator of controlled RKHS norm.}
The coefficient vector of \(h_s^0\) in the empirical RKHS basis is \(O_s^\top b_s\), so its norm is at most \(R\).
Also
\[
h_s^0-g_s=A_sO_s^\top b_s,\qquad
\|h_s^0-g_s\|_{\mathcal H_\kappa}\le CR\eta_N/r_s.
\]
Combine \(\|h\|_{L_2}^2\le\lambda_1\|h\|_{\mathcal H_\kappa}^2\), Lemma~\ref{lem:approx_rate}, and the squared triangle inequality to obtain \eqref{eq:transport_approx}.

\emph{Step 5: the deterministic population residual.}
Let \(r_{0,s}=f-g_s\). Population orthogonality implies
\[
\E[\varphi_s(X)r_{0,s}(X)]=0.
\]
Its \(L_2\) norm is bounded uniformly in \(s\), since \(\|g_s\|_{L_2}\le R\).
Using independence only for these population-feature summands gives
\begin{equation}\label{eq:population_residual_score}
\E\left\|\frac1n\sum_i\varphi_s(X_i)r_{0,s}(X_i)\right\|_2^2
\le\frac{M_\phi^2s}{n}\|r_{0,s}\|_{L_2}^2.
\end{equation}
For the alignment part of \eqref{eq:feature_perturb_split}, put
\[
m_n=\frac1n\sum_i k_{X_i}r_{0,s}(X_i).
\]
The orthogonality of \(f-g\) to \(\mathcal H_\kappa\) gives
\[
\E m_n=\mathscr L(g-g_s),\qquad
\|\E m_n\|_{\mathcal H_\kappa}\le R\lambda_{s+1},
\qquad
\E\|m_n-\E m_n\|_{\mathcal H_\kappa}^2\le C/n.
\]
The mean and centered second moment give
\[
\|m_n\|_{\mathcal H_\kappa}^2
\le2\|m_n-\E m_n\|_{\mathcal H_\kappa}^2+2\|\E m_n\|_{\mathcal H_\kappa}^2
=O_P(n^{-1}+R^2\lambda_{s+1}^2).
\]
Since $\lambda_{s+1}^2/\lambda_s\le\lambda_s$, multiplication by $\eta_N^2/(\lambda_s r_s^2)$ gives the next display. No independence between $A_s$ and $m_n$ is used.
\begin{equation}\tag{T3b}\label{eq:aligned_residual_score}\begin{aligned}
\|\Lambda_s^{-1/2}O_sA_s^*m_n\|_2^2
&\le\frac{C\eta_N^2}{\lambda_sr_s^2}\|m_n\|_{\mathcal H_\kappa}^2\\
&=O_P\!\left[
\frac{d_N^2}{r_s^2}\left(\lambda_s+\frac1{n\lambda_s}\right)\right]
=O_P(d_N^2/r_s^2).
\end{aligned}\end{equation}
The last step uses \(\lambda_s\le1\) and \(n\lambda_s\ge c_*>0\).
For the proxy part, Cauchy--Schwarz gives
\[
\left\|\frac1n\sum_i(k_{\widehat X_i}-k_{X_i})r_{0,s}(X_i)\right\|_{\mathcal H_\kappa}^2
\le L_\kappa^2\bar D_n\,\frac1n\sum_i r_{0,s}(X_i)^2
=O_P(u_{2\beta_\kappa}).
\]
After multiplying by \(\Lambda_s^{-1/2}O_sV_s^*\), its squared contribution is \(O_P(u_{2\beta_\kappa}/\lambda_s)\).
Since \(r_s\le\lambda_s\le1\), we have \(r_s^2\le\lambda_s\) and this is at most \(O_P(C_{\beta',\theta}/r_s^2)\).
Together with \eqref{eq:population_residual_score},
\[
\left\|\frac1nW^\top(r_{0,s}(X_i))_{i=1}^n\right\|_2^2
=O_P(s/n+d_N^2/r_s^2).
\]

\emph{Step 6: the other two residual pieces.}
Write
\[
r_i=r_{0,s}(X_i)+q_i+v_i,\quad
q_i=g_s(X_i)-h_s^0(X_i),\quad
v_i=h_s^0(X_i)-h_s^0(\widehat X_i).
\]
The pointwise RKHS bounds established above imply
\[
\|q\|_n^2\le CR^2\eta_N^2/r_s^2,\qquad
\|v\|_n^2\le R^2L_\kappa^2\bar D_n.
\]
For any vector \(a\in\R^n\), on the Gram event,
\[
\left\|G_n^{-1/2}\frac{Z^\top a}{n}\right\|_2^2
=\frac1n\|P_Za\|_2^2\le\|a\|_n^2,
\]
where \(P_Z=Z(Z^\top Z)^{-1}Z^\top\).
For the vector \(r_0=(r_{0,s}(X_i))_{i=1}^n\), the same quadratic form equals
\[
\left(\frac{W^\top r_0}{n}\right)^\top
(W^\top W/n)^{-1}
\left(\frac{W^\top r_0}{n}\right)
\le2\left\|\frac{W^\top r_0}{n}\right\|_2^2.
\]
Applying \(\|a+b+c\|_2^2\le3(\|a\|_2^2+\|b\|_2^2+\|c\|_2^2)\) proves \eqref{eq:full_residual_score}.
This proof needs only \(f\in L_2\), not an unstated bound on \(\|f\|_\infty\).
\end{proof}

\begin{lemma}[Estimation in the learned feature space]\label{lem:within_space}
Let \(Z\) have rows \(\widehat\Phi_s(\widehat X_i)^\top\), and define
\[
G_n=Z^\top Z/n,\qquad
\Sigma_s=\E_*[\widehat\Phi_s(X)\widehat\Phi_s(X)^\top].
\]
On any covariate-measurable event with \(G_n\succ0\) and \(\Sigma_s\preceq4G_n\), take a covariate-measurable comparator \(h_0=\sum_ja_j\widehat\phi_j\) and put \(r_i=f(X_i)-h_0(\widehat X_i)\).
Under the conditional noise bound in Assumption~\ref{assumption:a1}, with conditional probability at least \(1-\delta\),
\[
\|\widehat g_{s,\xi}-h_0\|_{L_2}^2
\le48\bigg\{
16\tau^2\frac{s+\log(1/\delta)}{n}
+\left\|G_n^{-1/2}\frac{Z^\top r}{n}\right\|_2^2
+\frac{\xi}{4}\|a\|_2^2\bigg\}.
\]
Here the conditioning includes all latent and proxy covariates, while the test norm holds the fitted predictor fixed.
\end{lemma}

\begin{proof}
\emph{Step 1: the exact normal equation.}
The decomposition \(Y=Za+r+\mu\) gives
\[
(G_n+\xi I)(\widehat\gamma-a)
=Z^\top\mu/n+Z^\top r/n-\xi a.
\]
No independence of the rows of \(Z\) is used.

\emph{Step 2: an empirical-norm bound for a score.}
For any \(b\in\R^s\), diagonalize \(G_n\). Since
\(t/(t+\xi)^2\le1/t\) for \(t>0\),
\[
\|G_n^{1/2}(G_n+\xi I)^{-1}b\|_2^2
\le b^\top G_n^{-1}b.
\]

\emph{Step 3: ridge bias of the comparator signal.}
The same diagonalization and the elementary maximum
\[
\sup_{t\ge0}\frac{\xi^2t}{(t+\xi)^2}=\frac{\xi}{4}
\]
give
\[
\|G_n^{1/2}(G_n+\xi I)^{-1}\xi a\|_2^2
\le(\xi/4)\|a\|_2^2.
\]
For \(\xi=0\), the term is zero.
Thus the squared empirical norm is at most three times the sum of the noise-score, residual-score, and displayed ridge terms.

\emph{Step 4: projected response noise.}
Let \(\Psi\) have orthonormal columns spanning \(\operatorname{col}(Z)\).
Put $\mathcal C_N=\sigma(U_{1:N},\widehat U_{1:N})$. The iid triple model makes the response noises conditionally independent given $\mathcal C_N$. For any unit $v\in\R^s$, let $a=\Psi v$. Since $\sum_i a_i^2=1$, \eqref{eq:conditional_noise_full} gives
\begin{align*}
\E[e^{t\sum_i a_i\mu_i}\mid\mathcal C_N]
&=\prod_i\E[e^{ta_i\mu_i}\mid\mathcal C_N]\\
&\le\prod_i e^{\tau^2t^2a_i^2/2}
=e^{\tau^2t^2/2}.
\end{align*}
This justifies the scalar tail bound even though $\Psi$ depends on all covariates.

A \(1/2\)-net of the unit sphere has at most \(5^s\) elements, and
\(\|x\|_2\le2\max_{v\ \mathrm{in\ the\ net}}v^\top x\).
The scalar sub-Gaussian tail and a union bound therefore give
\[
\Pb(\|\Psi^\top\mu\|_2>t\mid U_{1:N},\widehat U_{1:N})
\le5^s\exp\{-t^2/(8\tau^2)\}.
\]
It follows that, with conditional probability at least \(1-\delta\),
\begin{equation}\tag{N1}\label{eq:projected_noise_bound}
\left\|G_n^{-1/2}\frac{Z^\top\mu}{n}\right\|_2^2
=\frac1n\|\Psi^\top\mu\|_2^2
\le16\tau^2\frac{s+\log(1/\delta)}n.
\end{equation}
The net argument applies to a design-dependent \(\Psi\) precisely because the noise bound is conditional.

\emph{Step 5: transfer the empirical norm to test risk.}
For every coefficient vector \(v\), the Gram event gives
\[
v^\top\Sigma_sv\le4v^\top G_nv.
\]
Apply this to \(v=\widehat\gamma-a\) and combine the previous bounds.
The constant \(48\) is a harmless enlargement of the resulting \(12\).
The argument bounds an excess function error directly, so no irreducible risk is multiplied by a hidden oracle constant.
\end{proof}

\begin{technicallemma}[Derived radius of the fitted predictor]\label{tl:radius}
On the Gram event in Technical Lemma~\ref{tl:gram_score}, for every \(\xi\ge0\),
\[
\|\widehat g_{s,\xi}\|_{\mathcal H_\kappa}^2
=\|\widehat\gamma\|_2^2
\le\frac{2}{\lambda_s}\frac1n\sum_iY_i^2.
\]
Consequently a deterministic \(B_s\ge R\) with \(B_s^2\lesssim\lambda_s^{-1}\lesssim r_s^{-2}\) bounds the fitted norm with probability tending to one.
\end{technicallemma}

\begin{proof}
Use the singular value decomposition \(Z/\sqrt n=U\,\operatorname{diag}(t_j)V^\top\).
For \(t_j>0\),
\[
\frac{t_j^2}{(t_j^2+\xi)^2}\le\frac1{t_j^2}
\le\frac1{\lambda_{\min}(G_n)}.
\]
The ridge coefficient formula therefore implies
\[
\|\widehat\gamma\|_2^2
\le\frac{\|Y\|_2^2/n}{\lambda_{\min}(G_n)}
\le\frac{2\|Y\|_n^2}{\lambda_s}.
\]
Here \(t_j^2\), not \(t_j^4\), is the corresponding Gram eigenvalue.
Since \(\E Y^2<\infty\), the weak law gives \(\|Y\|_n^2\to\E Y^2\) in probability.
Choose, for example,
\(B_s^2=4(\E Y^2+R^2+1)/\lambda_s\).
Finally \(r_s\le\lambda_s\le1\) implies \(1/\lambda_s\le1/r_s^2\).
\end{proof}

\subsubsection{Proof of Theorem~\ref{prop:error-estimator-best}}

\begin{proof}[Proof of Theorem~\ref{prop:error-estimator-best}]
\emph{Step 1: the common good event.}
Let \(\mathcal E_{\rm gap}\) be the gap event in Assumption~\ref{assumption:a5}.
Let \(\mathcal E_{\rm Gram}\) be the event on which both normalized covariance matrices in Technical Lemma~\ref{tl:gram_score} are between \(\tfrac12 I_s\) and \(2I_s\).
Its proof shows that \eqref{eq:secondstage_suff} gives
\(\Pb(\mathcal E_{\rm gap}\cap\mathcal E_{\rm Gram})\to1\).
Indeed
\[
\Pb((\mathcal E_{\rm gap}\cap\mathcal E_{\rm Gram})^c)
\le\Pb(\mathcal E_{\rm gap}^c)+\Pb(\mathcal E_{\rm Gram}^c).
\]

\emph{Step 2: a feasible comparison function.}
Use \(h_s^0=V_sO_s^\top b_s\) from Technical Lemma~\ref{tl:gram_score}.
It belongs to \(\widehat{\mathcal F}_s\) and has RKHS norm at most \(R\).
Separating truncation and transported-subspace error gives
\[
\|h_s^0-g\|_{L_2}^2
\le 2R^2\lambda_{s+1}+CR^2\eta_N^2/r_s^2.
\]
The first term is \(O(s^{-q})\) by Assumption~\ref{assumption:a3}; the second is
\(O_P(d_N^2/r_s^2)\) by Lemma~\ref{lem:op_dev}.

\emph{Step 3: the second-stage coefficient error.}
Lemma~\ref{lem:within_space}, with \(a=O_s^\top b_s\), and
\eqref{eq:full_residual_score} give
\[
\|\widehat g-h_s^0\|_{L_2}^2
=O_P\!\left(\frac{s}{n}+\frac{d_N^2}{r_s^2}+\xi R^2\right).
\]
For precision, the conditional noise tail in that lemma implies an \(O_P(s/n)\) statement by taking any fixed confidence tolerance; \(s\ge1\) absorbs its fixed \(\log(1/\delta)\).
The other score bounds already are unconditional \(O_P\) bounds and may be intersected with the common good event.

\emph{Step 4: assemble the error budget.}
For every realization on that event,
\[
\|\widehat g-g\|_{L_2}^2
\le2\|\widehat g-h_s^0\|_{L_2}^2
  +2\|h_s^0-g\|_{L_2}^2.
\]
Consequently,
\[
\|\widehat g-g\|_{L_2}^2
=O_P\!\left(s^{-q}+\frac{N^{-1}+C_{\beta',\theta}}{r_s^2}
                  +\frac{s}{n}+\xi\right).
\]
An inequality established on an event of probability tending to one yields the same \(O_P\) conclusion: for any threshold \(M a_n\), its tail probability is at most the probability of the complementary event plus the tail probability of the displayed bound.
This is $O_P(R_{s,n,N}^2)$ for the deterministic expression \eqref{eq:R_bound}, proving Theorem~\ref{prop:error-estimator-best}.
Finally Technical Lemma~\ref{tl:radius} proves the fitted-radius conclusion; it is derived rather than assumed.
\end{proof}

\subsubsection{Proof of Theorem~\ref{thm:error-regression}}

\begin{proof}[Proof of Theorem~\ref{thm:error-regression}]

In this proof the training sample is fixed and the expectation is over a fresh latent \(X\); thus the random function \(\widehat g-g\) is a fixed element of \(\mathcal H_\kappa\) when applying population orthogonality.

Since
\[
g
\in
\argmin_{h\in\mathcal H_\kappa}
\E\!\left[(f(X)-h(X))^2\right],
\]
and $\mathcal H_\kappa$ is a linear space, first-order
optimality implies
\[
\E\!\left[
\{f(X)-g(X)\}h(X)
\right]
=
0,
\qquad
h\in\mathcal H_\kappa.
\]
Indeed, for any $h\in\mathcal H_\kappa$, the function
\[
t\mapsto
\E\!\left[
\{f(X)-g(X)-t h(X)\}^2
\right]
\]
is minimized at $t=0$, and differentiating at zero gives
the displayed orthogonality relation.

Since both $\widehat g$ and $g$ belong to
$\mathcal H_\kappa$,
\[
\widehat g-g\in\mathcal H_\kappa.
\]
Therefore,
\[
\E\!\left[
\{f(X)-g(X)\}
\{\widehat g(X)-g(X)\}
\right]
=
0.
\]
Expanding the squared regression error gives
\begin{align*}
\E\!\left[
(\widehat g(X)-f(X))^2
\right]
&=
\E\!\left[
(\widehat g(X)-g(X))^2
\right]
+
\E\!\left[
(g(X)-f(X))^2
\right]
\\
&=
\E\!\left[
(\widehat g(X)-g(X))^2
\right]
+
\epsilon^2.
\end{align*}
Applying Theorem~\ref{prop:error-estimator-best} yields
\[
\E\!\left[
(\widehat g(X)-f(X))^2
\right]
=
\epsilon^2+O_P(R_{s,n,N}^2),
\]
which proves \eqref{eq:reg_bound}.
\end{proof}

\subsubsection{Proof of Theorem~\ref{thm:error-prediction}}

\begin{proof}[Proof of Theorem~\ref{thm:error-prediction}]
All expectations below hold the training sample fixed.
On the radius event of Technical Lemma~\ref{tl:radius}, Lemma~\ref{lem:proxy_transfer} gives
\[
|\mathcal R_{\widehat X}(\widehat g)-\mathcal R_X(\widehat g)|
\le C(1+B_s^2)C_{\beta',\theta}
\le C_2r_s^{-2}C_{\beta',\theta}.
\]
For the latent risk, expand
\[
(Y-\widehat g(X))^2
=(f(X)-\widehat g(X))^2+\mu^2+2\mu(f(X)-\widehat g(X)).
\]
The last term has conditional test expectation zero, because the fresh noise has mean zero given \(X\). Thus Theorem~\ref{thm:error-regression} gives
\[
\mathcal R_X(\widehat g)
=\epsilon^2+\tau_0^2+\|\widehat g-g\|_{L_2}^2.
\]
To make the constant enlargement explicit, write
\[
U_{s,n,N}=s^{-q}+r_s^{-2}(N^{-1}+C_{\beta',\theta})+s/n+\xi.
\]
On any event on which \(\|\widehat g-g\|_{L_2}^2\le C_1U_{s,n,N}\), we have
\begin{align*}
\mathcal R_{\widehat X}(\widehat g)
&\le\mathcal R_X(\widehat g)+C_2r_s^{-2}C_{\beta',\theta}\\
&\le\epsilon^2+\tau_0^2+(C_1+C_2)U_{s,n,N}.
\end{align*}
The latent risk identity and the absolute transfer bound give
\[
|\mathcal R_{\widehat X}(\widehat g)-\epsilon^2-\tau_0^2|
\le\mathcal E(\widehat g)+|\mathcal R_{\widehat X}(\widehat g)-\mathcal R_X(\widehat g)|
=O_P(R_{s,n,N}^2).
\]
The two irreducible constants have coefficient one. The stochastic remainder need not be nonnegative, because the test predictor is evaluated at $\widehat X$ rather than $X$.
\end{proof}

\subsubsection{Proof of Corollary~\ref{cor:error-rate}}

\begin{proof}[Proof of Corollary~\ref{cor:error-rate}]
For fixed positive \(A,B\), the continuous tradeoff
\(F(s)=As^{-q}+Bs/n\) has derivative
\[
F'(s)=-qAs^{-q-1}+B/n.
\]
Its minimizer is \(s_0=(qAn/B)^{1/(q+1)}\).
Choosing a neighboring positive integer yields \(s_n\asymp n^{1/(q+1)}\).
Thus
\[
s_n^{-q}\asymp s_n/n\asymp n^{-q/(q+1)}.
\]
The operator and ridge negligibility assumptions make the remaining two contributions in \eqref{eq:R_bound} smaller than this order, proving the rate.

The second-stage conditions added to the theorem are compatible with this tuning, rather than an additional oracle-rate obstruction.
Indeed \(\lambda_{s_n}\asymp s_n^{-q}\) gives
\[
\frac{s_n^2}{n}\asymp n^{(1-q)/(q+1)}\to0,\qquad
n\lambda_{s_n}\asymp n^{1/(q+1)}\to\infty,
\]
and
\[
\frac{N^{-1}+C_{\beta',\theta}}{\lambda_{s_n}r_{s_n}^2}
=o\!\left(\frac{s_n^{-q}+s_n/n}{\lambda_{s_n}}\right)=o(1).
\]
The last equality uses the two-sided spectral condition at the chosen index.
The regression and prediction conclusions then follow from Theorems~\ref{thm:error-regression} and \ref{thm:error-prediction}.
\end{proof}

\paragraph{Explicit eigengap-dependent sampling rates.}\label{app:gap_rates}
To make the dependence on the eigengap explicit, suppose
that the relevant local eigengap satisfies
$r_s\gtrsim s^{-\alpha}$ for some $\alpha\ge0$. Then
\[
\frac{1}{r_s^2}
\left\{
\frac{1}{N}+C_{\beta',\theta}
\right\}
\lesssim
s^{2\alpha}
\left\{
\frac{1}{N}+C_{\beta',\theta}
\right\}.
\]
At $s\asymp n^{1/(q+1)}$, condition~\eqref{eq:negligible_operator_condition} is implied by
\[
\frac{1}{N}+C_{\beta',\theta}
=
o\!\left(
n^{-(q+2\alpha)/(q+1)}
\right).
\]
Equivalently, it suffices that
\[
N\gg n^{(q+2\alpha)/(q+1)},
\qquad
C_{\beta',\theta}
=
o\!\left(
n^{-(q+2\alpha)/(q+1)}
\right).
\]
Under the regular-spacing heuristic
$r_s\asymp\Delta_s\asymp s^{-(q+1)}$, this becomes
\[
N\gg n^{(3q+2)/(q+1)}.
\]

\paragraph{Fixed proxy scale and the scope of the bound.}\label{app:fixed_proxy}
Condition~\eqref{eq:negligible_ridge_condition} permits \(\xi=0\) on the Gram event. The oracle-rate conclusion is conditional on decreasing operator and proxy errors. A fixed nonzero proxy scale can violate \eqref{eq:secondstage_suff}; the asymptotic theorem is then not an unconditional fixed-proxy guarantee. The finite-design inequality in Lemma~\ref{lem:within_space} and the uniform proxy-transfer lemma remain valid on their explicit comparison and radius events. A fixed-proxy conclusion requires verifying those quantitative events in that regime. Growing-bag distribution regression is one setting in which the proxy error decreases with the sampling design.

\paragraph{The labeled-only comparison and a sufficient truncation.}\label{app:labeled_only}
When \(N=n\) and proxy and ridge errors are negligible, the displayed expression is
\[
s^{-q}+s^{2\alpha}/n+s/n.
\]
Writing \(b_\alpha=\max\{1,2\alpha\}\), formal balancing gives
\[
s\asymp n^{1/(q+b_\alpha)},\qquad
n^{-q/(q+b_\alpha)}.
\]
In particular, \(\alpha=q+1\) gives the formal exponent \(q/(3q+2)\). These balances alone do not verify \eqref{eq:secondstage_suff}. For \(2\alpha\ge1\), a sufficient truncation is
\[
s\asymp(n/L_n)^{1/(q+2\alpha)},\qquad
L_n\to\infty,\quad L_n=n^{o(1)},
\]
provided proxy and ridge terms are negligible at that truncation. Indeed,
\[
\frac{N^{-1}}{\lambda_s r_s^2}
\lesssim\frac{s^{q+2\alpha}}n\asymp L_n^{-1}\to0,
\]
using \(\lambda_s\asymp s^{-q}\), while \(s^2/n\to0\) and \(n\lambda_s\to\infty\).
The leading bias becomes
\[
n^{-q/(q+2\alpha)}L_n^{q/(q+2\alpha)},
\]
which has the stated polynomial exponent with a slowly varying correction. This is an upper-bound comparison, not a minimax claim for labeled-only learning.

\subsubsection{Proof of Theorem~\ref{thm:cv_rate}}

\paragraph{Precise validation conditions and bound.}\label{app:validation_conditions}
The concise main-text Theorem~\ref{thm:cv_rate} uses the following full formulation.

Let \(\mathcal G=\{\widehat g_{s,\xi}\}\) be the finite candidate set, \(M=|\mathcal S||\Xi|\), and
\[
\Delta(h)=\E_*[(h(X)-g(X))^2].
\]
Suppose a common training event $\mathcal E_{\rm tr}$ has probability tending to one and, on it, for a fixed $C_{\rm tr}<\infty$,
\[
\Delta(\widehat g_{s,\xi})\le C_{\rm tr}R_{s,n_{\rm tr},N}^2
\quad\text{for every }(s,\xi)\in\mathcal S\times\Xi.
\]
The fitted norm bounds below must also hold on this common event. Fixed-candidate $O_P$ bounds do not by themselves prove this uniform event for a growing grid. 
Assume, additionally, the validation concentration condition
\begin{equation}\label{eq:validation_condition}
\begin{aligned}
\big|\widehat L_{\rm va}(h)
 -\{\mathcal R_{\widehat X}(h)-\mathcal R_{\widehat X}(g)\}\big|
\le {}&\tfrac14\Delta(h)\\
&+C_{\rm va}\frac{\log(M/\delta)}{n_{\rm va}}
\end{aligned}
\end{equation}
simultaneously for \(h\in\mathcal G\), with probability at least \(1-\delta\) conditional on the training good event. Here
\[
\widehat L_{\rm va}(h)=\frac1{n_{\rm va}}
\sum_{i\in\mathcal I_{\rm va}}
\{(Y_i-h(\widehat X_i))^2-(Y_i-g(\widehat X_i))^2\}.
\]
If \(\max_{h\in\mathcal G}\|h\|_{\mathcal H_\kappa}\le B_{\max}\), let
\[
\Gamma_{\mathcal G}
=C(1+B_{\max}^2+R^2)C_{\beta',\theta}.
\]
Then, on the training good event and the event in \eqref{eq:validation_condition},
\[
\begin{aligned}
\Delta(\widehat g_{\rm sel})
\le{}&\tfrac53\min_{h\in\mathcal G}\Delta(h)\\
&+\tfrac83\left\{\Gamma_{\mathcal G}
+C_{\rm va}\frac{\log(M/\delta)}{n_{\rm va}}\right\}.
\end{aligned}
\]
Consequently, if \(n_{\rm tr}\asymp n_{\rm va}\asymp n\), then for a polynomial-size grid containing an oracle-rate candidate, the selected predictor has the oracle rate provided \eqref{eq:validation_condition} holds with bounded \(C_{\rm va}\) and \(\Gamma_{\mathcal G}=O_P(n^{-q/(q+1)})\).

\begin{proof}[Proof of Theorem~\ref{thm:cv_rate}]
\emph{Step 1: specify the two contrasts.}
Write
\[
L(h)=\mathcal R_{\widehat X}(h)-\mathcal R_{\widehat X}(g),
\qquad
\Delta(h)=\E_*[(h(X)-g(X))^2].
\]
The reference \(g\) is used only in the analysis; subtracting its empirical validation loss does not change the minimizer.
For a finite candidate class whose RKHS norms are at most \(B_{\max}\), Lemma~\ref{lem:proxy_transfer} and latent orthogonality imply
\[
|L(h)-\Delta(h)|
\le C(1+B_{\max}^2+R^2)C_{\beta',\theta}
=\Gamma_{\mathcal G}.
\]
This is a uniform cost over the grid, so it is not silently replaced by the proxy cost of the best single candidate.

\emph{Step 2: use exactly the assumed concentration event.}
Put \(b=C_{\rm va}\log(M/\delta)/n_{\rm va}\).
On the simultaneous event in \eqref{eq:validation_condition}, for every candidate,
\[
|\widehat L_{\rm va}(h)-\Delta(h)|
\le\tfrac14\Delta(h)+b+\Gamma_{\mathcal G}.
\]
The factor \(\log M\) is the finite-class union-bound cost already specified in that condition; its constant \(C_{\rm va}\) is not inferred merely from sub-Gaussian response noise.

\emph{Step 3: use the validation minimization property.}
Let \(h_*\) minimize \(\Delta\) over the finite class and put \(\widehat h=\widehat g_{\rm sel}\).
Since \(\widehat L_{\rm va}(\widehat h)\le\widehat L_{\rm va}(h_*)\),
\begin{align*}
\tfrac34\Delta(\widehat h)
&\le\widehat L_{\rm va}(\widehat h)+b+\Gamma_{\mathcal G}\\
&\le\widehat L_{\rm va}(h_*)+b+\Gamma_{\mathcal G}\\
&\le\tfrac54\Delta(h_*)+2b+2\Gamma_{\mathcal G}.
\end{align*}
Divide by \(3/4\) to obtain the stated constants \(5/3\) and \(8/3\).
If the training good event fails with probability \(\eta_{\rm tr}\), the unconditional probability of this implication is at least \(1-\delta-\eta_{\rm tr}\).

\emph{Step 4: the conditional rate claim.}
An oracle-rate candidate makes \(\min_h\Delta(h)=O_P(n^{-q/(q+1)})\).
For a polynomial-size grid, \(n_{\rm va}\asymp n\), and bounded \(C_{\rm va}\),
\[
\frac{\log(M/\delta)}{n_{\rm va}}
=O(\log n/n)=o(n^{-q/(q+1)})
\]
at any fixed confidence tolerance.
The additional stated assumption on \(\Gamma_{\mathcal G}\) gives the rate.

\paragraph{Remaining proof obligation for the original validation protocol.}
This proof establishes the deterministic selection implication on \eqref{eq:validation_condition}; it does not prove that condition for eigenfeatures using the validation covariates.
Conditioning on all proxy covariates fixes the validation design. Independent validation labels then concentrate around a design-conditional average, which must still be related to fresh-sample risk.
To remove \eqref{eq:validation_condition} as an extra assumption, one must prove that relation uniformly over the grid, with a bounded concentration constant and the displayed proxy cost. Independent feature-learning and validation samples, or a valid conditional-design/generalization argument for reused covariates, are possible routes. Label independence alone is insufficient for that missing relation.
\end{proof}

\subsubsection{Proof of Theorem~\ref{thm:approx_eig}}

\paragraph{Precise approximate-feature conditions.}\label{app:approx_conditions}
The approximate functions are orthonormal in \(\mathcal H_\kappa\), so their coefficient penalty equals their RKHS penalty. 
They are measurable functions of the proxy covariates and an auxiliary seed $\omega_{\rm alg}$ independent of the entire sample. Equivalently for the noise argument, it suffices that $\omega_{\rm alg}\perp\mu_{1:n}\mid\mathcal C_N$ and that the conditional bound \eqref{eq:conditional_noise_full} is preserved after adding the seed to the conditioning sigma-field. 
Orthonormalization of their RKHS Gram matrix can enforce this normalization. In addition to \ref{assump:spec_err}, assume \(\varepsilon_{\rm spec}^2/\lambda_s\to0\). The proof below verifies both the approximate Gram comparison and residual-score bound from these conditions and \eqref{eq:secondstage_suff}; an approximate-space regression inequality is not assumed separately.

\paragraph{From an operator approximation to a projector bound.}
Define $S_N:\mathcal H_\kappa\to\R^N$ by $(S_Nh)_i=N^{-1/2}h(\widehat X_i)$. Then
\[
S_N^*a=N^{-1/2}\sum_i a_i k_{\widehat X_i},\qquad
\widehat{\mathscr L}_N=S_N^*S_N,\qquad \mathbf K/N=S_NS_N^*.
\]
These identities explain the nonzero eigenvalue correspondence, but do not identify every matrix approximation with an operator on $\mathcal H_\kappa$. A sufficient solver-specific certificate is a compact self-adjoint $\widetilde{\mathscr L}$ on this same RKHS whose leading $s$-dimensional range is the returned, RKHS-orthonormalized feature space and such that
\[
\|\widetilde{\mathscr L}-\widehat{\mathscr L}_N\|_{\rm op}\le a,
\qquad a\le\widehat\Delta_s/3,\qquad
\widehat\Delta_s=\widehat\lambda_s-\widehat\lambda_{s+1}>0.
\]
Applying the two off-diagonal Sylvester bounds from Lemma~\ref{lem:eigenspace-bound}, now to these two operators, yields
\[
\|\widetilde\Pi_s^{\mathcal H}-\widehat\Pi_s^{\mathcal H}\|_{\rm op}
\le\frac{2a}{\widehat\Delta_s-a}\le\frac{3a}{\widehat\Delta_s}.
\]
If the certificate holds only on a solver event of failure probability $p_{\rm alg}$, the finite-probability conclusion must add $p_{\rm alg}$ to its failure probability. The stated asymptotic result follows if $p_{\rm alg}\to0$. This is a sufficient conditional connection, not a solver guarantee inferred from a bare matrix-norm error.

\begin{proof}[Proof of Theorem~\ref{thm:approx_eig}]
\emph{Step 1: the approximation calculation.}
Write \(P=\widehat\Pi_s^{\mathcal H}\), \(Q=\widetilde\Pi_s^{\mathcal H}\).
For \(h=QPg\), idempotence \(P^2=P\) gives
\[
h-g=(Q-P)Pg+(Pg-g),\qquad \|h\|_{\mathcal H_\kappa}\le R.
\]
Hence
\begin{align*}
\|h-g\|_{L_2}^2
&\le2\|Pg-g\|_{L_2}^2
    +2\lambda_1\|(Q-P)Pg\|_{\mathcal H_\kappa}^2\\
&\le2\|Pg-g\|_{L_2}^2+2\lambda_1R^2\varepsilon_{\rm spec}^2.
\end{align*}
Also \(\xi\|h\|_{\mathcal H_\kappa}^2\le\xi R^2\).
This explains the order of the extra approximation term, but a regression bound also requires stability of the approximate features.

\emph{Step 2: align normalized approximate bases.}
Let \(\widetilde V_s:\R^s\to\mathcal H_\kappa\) be the isometry defined by the orthonormalized approximate functions.
The polar alignment of \(V_s^*\widetilde V_s\) gives an orthogonal \(T_s\) such that
\[
\|\widetilde V_s-V_sT_s\|_{\rm op}
\le\sqrt2\,\|Q-P\|_{\rm op}\le\sqrt2\,\varepsilon_{\rm spec}.
\]
To verify this without assuming a principal-angle formula, write $V_s^*\widetilde V_s=L\operatorname{diag}(c_j)R^\top$ with $0\le c_j\le1$ and choose $T_s=LR^\top$. For $B=\widetilde V_s-V_sT_s$,
\[
B^*B=2I_s-2R\operatorname{diag}(c_j)R^\top,
\qquad
\widetilde V_s^*(I-P)\widetilde V_s=I_s-R\operatorname{diag}(c_j^2)R^\top.
\]
Since $(I-P)\widetilde V_s=(Q-P)\widetilde V_s$, it follows that
\[
\|B\|_{\rm op}^2=2\max_j(1-c_j)
\le2\max_j(1-c_j^2)
=2\|(I-P)\widetilde V_s\|_{\rm op}^2
\le2\|Q-P\|_{\rm op}^2.
\]
After composing with the population alignment, we obtain an orthogonal \(\widetilde O_s\) with
\[
\|\widetilde V_s-U_s\widetilde O_s\|_{\rm op}
\le C\eta_N/r_s+\sqrt2\,\varepsilon_{\rm spec}
=:a_N.
\]

\emph{Step 3: verify the approximate Gram comparison.}
Repeat \eqref{eq:feature_perturb_split} with \(\widetilde V_s,\widetilde O_s\).
The squared normalized training-feature discrepancy is at most
\[
\frac{C}{\lambda_s}(a_N^2+\bar D_n).
\]
It is \(o_P(1)\) by \eqref{eq:secondstage_suff} and
\(\varepsilon_{\rm spec}^2/\lambda_s\to0\).
The same is true of the fresh latent discrepancy.
The matrix expansion in Technical Lemma~\ref{tl:gram_score} therefore gives its Gram and test-norm comparisons for the approximate features.
Orthonormalization is used here and in the equality between coefficient and RKHS penalties.

\emph{Step 4: transfer the residual-score calculation.}
Use the comparator \(\widetilde h_s^0=\widetilde V_s\widetilde O_s^\top b_s\).
Its RKHS norm is at most \(R\), and
\[
\|\widetilde h_s^0-g\|_{L_2}^2
\le C\{s^{-q}+\eta_N^2/r_s^2+\varepsilon_{\rm spec}^2\}.
\]
The population-feature score in \eqref{eq:population_residual_score} is unchanged.
The alignment-score calculation now has the upper bound
\[
\frac{Ca_N^2}{\lambda_s}\|m_n\|_{\mathcal H_\kappa}^2
=O_P(a_N^2),
\]
using \(n\lambda_s\ge c_*\) as before.
The proxy piece is still \(O_P(u_{2\beta_\kappa}/\lambda_s)\).
The remaining comparator and evaluation pieces have empirical squared norms
\(O_P(a_N^2)\) and \(O_P(u_{2\beta_\kappa})\), respectively.
Thus the analogue of \eqref{eq:full_residual_score} is
\[
O_P\!\left(s/n+d_N^2/r_s^2+\varepsilon_{\rm spec}^2\right).
\]

\emph{Step 5: fit and combine.}
Apply the explicitly proved finite-design Lemma~\ref{lem:within_space} to the approximate orthonormal basis.
Combining its error with the displayed approximation bound yields
\[
\|\widetilde g-g\|_{L_2}^2
=O_P\!\left(s^{-q}+d_N^2/r_s^2+s/n+\xi+\varepsilon_{\rm spec}^2\right).
\]
This is exactly $O_P(R_{s,n,N}^2+\varepsilon_{\rm spec}^2)$ with the deterministic rate expression in \eqref{eq:R_bound}.
At the oracle index, \(\lambda_s\asymp n^{-q/(q+1)}\).
The assumption \(\varepsilon_{\rm spec}=o(n^{-q/(2q+2)})\) makes both
\(\varepsilon_{\rm spec}^2=o(n^{-q/(q+1)})\) and
\(\varepsilon_{\rm spec}^2/\lambda_s=o(1)\).
\end{proof}

\subsection{Proofs for Section~\ref{sec:distribution-reg}}
\label{app:proofs_dist}

Throughout, \(c,C>0\) denote finite constants that may change from line to line and may depend only on fixed problem parameters, but never on \((n,N,s,\xi,m)\). We write \(A\lesssim B\) for \(A\le CB\). Let
\[
\beta'=\min\{\beta_\kappa,\beta_\phi\},
\qquad
a_m=\rho(m)^{\beta'}+\rho(m)^{2\beta'} .
\]

For \(0<t\le2\), conditional moment comparison and \ref{assumption:b4} give
\[
\E D(P,\widehat P)^t
\le\{\E D(P,\widehat P)^2\}^{t/2}
\le C_t\rho(m)^t.
\]
The squared first moments at exponents \(\beta_\phi,\beta_\kappa\), the \(2\beta_\kappa\)-moment, and its square root are therefore bounded by constants times
\(\rho(m)^{2\beta_\phi},\rho(m)^{2\beta_\kappa},
\rho(m)^{2\beta_\kappa},\rho(m)^{\beta_\kappa}\), respectively.
Since \(\rho(m)\le1\) eventually and both regularity exponents are at least \(\beta'\), the common envelope \(a_m\) controls all four quantities.

\subsubsection{Auxiliary lemmas}

\begin{lemma}[Bounded proxy spectral features]\label{lem:dist_feat_bound}
Under Assumption~\ref{assumption:b2}, for every \(s\ge1\),
\[
\|\widehat\Phi_s(\widehat P)\|_2\le 1
\]
almost surely.
\end{lemma}

\begin{proof}
Because \(\{\widehat\phi_j\}_{j\ge1}\) is orthonormal in \(\mathcal H_\kappa\),
\[
\|\widehat\Phi_s(\widehat P)\|_2^2
=
\sum_{j=1}^s \widehat\phi_j(\widehat P)^2
=
\sum_{j=1}^s
\langle \widehat\phi_j,\kappa(\widehat P,\cdot)\rangle_{\mathcal H_\kappa}^2
\le
\|\kappa(\widehat P,\cdot)\|_{\mathcal H_\kappa}^2
=
\kappa(\widehat P,\widehat P)
\le 1.
\]
\end{proof}

\begin{lemma}[Spectral approximation]\label{lem:dist_approx}
Assume Assumptions~\ref{assumption:b3} and \ref{assumption:b6}. Let \(\Pi_s\) be the \(L_2(\Pb_P)\)-orthogonal projection onto \(\mathrm{span}\{\phi_1,\ldots,\phi_s\}\) and set \(g_s=\Pi_sg\). Then
\[
\|g-g_s\|_{L_2(\Pb_P)}^2
\le
\lambda_{s+1}\|g\|_{\mathcal H_\kappa}^2
\lesssim
s^{-q}.
\]
\end{lemma}

\begin{proof}
Write \(g=\sum_jb_j\sqrt{\lambda_j}\phi_j+g_0\), with
\(g_0\in\ker\mathscr L\) and \(\sum_jb_j^2\le R^2\).
Then \(g_s=\sum_{j\le s}b_j\sqrt{\lambda_j}\phi_j\), and
\[
\|g-g_s\|_{L_2(\Pb_P)}^2
=\sum_{j>s}\lambda_jb_j^2
\le R^2\lambda_{s+1}\lesssim s^{-q}.
\]
Thus the generic proof specializes with precisely the assumptions stated here.
\end{proof}

\begin{lemma}\label{lem:prob-distance}
Under Assumption~\ref{assumption:b4}, for every $0<\beta\le1$,
\[
\frac1N\sum_{i=1}^N D(\widehat P_i,P_i)^\beta
=
\mathcal O_P\!\left(\rho(m)^\beta\right).
\]
\end{lemma}

\begin{proof}
For \(0<\beta\le1\), \ref{assumption:b4} and concavity give
\[
\E D(P,\widehat P)^\beta
\le\{\E D(P,\widehat P)\}^{\beta}
\le C_D^\beta\rho(m)^\beta.
\]
Linearity of expectation and Markov's inequality imply, for every \(M>0\),
\[
\Pb\!\left(\frac1N\sum_iD(P_i,\widehat P_i)^\beta
             >M\rho(m)^\beta\right)\le C_D^\beta/M.
\]
Choosing \(M\) for an arbitrary fixed confidence tolerance proves the \(O_P\) assertion.
The same reasoning for \(0<\beta\le2\), using the second moment instead of the first, supplies the powers needed in Gram and evaluation bounds.
\end{proof}

\begin{technicallemma}[A concrete KDE sufficient condition for (B4)]\label{tl:kde_moments}
Suppose \(D(p,q)=\|p-q\|_{L_2(\R^d)}\), and let \(K\ge0\) satisfy
\(\int K=1\), \(K\in L_2(\R^d)\).
For a deterministic bandwidth $b>0$ and a bag of \(m\) iid observations from \(p\), set
\[
\widehat p_b(x)=\frac1{mb^d}\sum_{\ell=1}^mK((x-Z_\ell)/b).
\]
Assume, for a constant $C_{\rm bias}$ independent of $p,b,m$, that
\[
\|\E[\widehat p_b\mid p]-p\|_{L_2}\le C_{\rm bias}b,
\qquad 0<b\le b_0,
\]
uniformly over the stated latent density class.
Also assume the outer kernel and its boundedness/regularity bounds are defined on both the latent densities and these KDE outputs.
Then, for every real \(k\ge1\),
\begin{equation}\tag{K1}\label{eq:kde_moment_bound}
\left\{\E[\|\widehat p_b-p\|_{L_2}^k\mid p]\right\}^{1/k}
\le C\left\{b+\frac{\sqrt k}{\sqrt{mb^d}}\right\}.
\end{equation}
Taking \(b\asymp m^{-1/(2+d)}\) verifies \ref{assumption:b4} with
\(\rho(m)\asymp m^{-1/(2+d)}\).
\end{technicallemma}

\begin{proof}
\emph{Step 1: separate bias and sampling error.}
Write \(q_Z(x)=b^{-d}K((x-Z)/b)\). Translation and change of variables give
\[
\|q_Z\|_{L_2}=b^{-d/2}\|K\|_{L_2}=:B_b.
\]
In the Hilbert space \(L_2(\R^d)\), the centered average
\(H=m^{-1}\sum_\ell(q_{Z_\ell}-\E[q_Z\mid p])\) satisfies
\[
\E[\|H\|_{L_2}^2\mid p]\le B_b^2/m,
\qquad
\E[\|H\|_{L_2}\mid p]\le B_b/\sqrt m.
\]

\emph{Step 2: a conditional tail for the Hilbert norm.}
Changing one observation changes \(\|H\|_{L_2}\) by at most \(2B_b/m\).
The bounded-differences inequality used here states that, for independent coordinates and a real function $F$ whose change under replacement of coordinate $i$ is at most $c_i$,
\[
\Pb(F-\E F>t)\le\exp\{-2t^2/\textstyle\sum_i c_i^2\}.
\]
Apply it conditionally on $p$, where the bag entries are independent, with $F=\|H\|_{L_2}$ and $c_i=2B_b/m$. Since $\sum_i c_i^2=4B_b^2/m$, this gives
\[
\Pb\!\left(\|H\|_{L_2}-\E[\|H\|_{L_2}\mid p]>t\mid p\right)
\le\exp\{-mt^2/(2B_b^2)\}.
\]
For \(T=(\|H\|_{L_2}-\E[\|H\|_{L_2}\mid p])_+\), integrate this tail:
\begin{align*}
\E[T^k\mid p]
&=\int_0^\infty kt^{k-1}\Pb(T>t\mid p)\,dt\\
&\le(2B_b^2/m)^{k/2}\Gamma(k/2+1).
\end{align*}
The elementary bound \(\Gamma(k/2+1)^{1/k}\le C\sqrt k\), \(k\ge1\), and the preceding mean bound imply
\(\{\E\|H\|_{L_2}^k\mid p\}^{1/k}\le C B_b\sqrt{k/m}\).
All constants are uniform in \(p\).

\emph{Step 3: add the deterministic bias and choose bandwidth.}
Minkowski's inequality gives the stated moment bound.
At \(b\asymp m^{-1/(2+d)}\), both \(b\) and \((mb^d)^{-1/2}\) have order \(m^{-1/(2+d)}\).
For one sufficient bias condition, if
\[
\|p(\cdot-t)-p\|_{L_2}\le L\|t\|,
\qquad
\int\|u\|K(u)\,du<\infty,
\]
uniformly in \(p\), then
\[
\|K_b*p-p\|_{L_2}
\le\int K(u)\|p(\cdot-bu)-p\|_{L_2}\,du
\le Lb\int\|u\|K(u)\,du.
\]
On a bounded domain, boundary correction or a compatible extension must supply the stated bias bound; smoothness in the interior alone does not address the boundary.
This Hilbert-space argument proves every required moment directly and avoids an invalid \(L_k\)-to-\(L_2\) step when \(k<2\).
\end{proof}

\begin{lemma}[Proxy operator deviation]\label{lem:dist_op_dev}
Assume the bag-wise iid model and Assumptions~\ref{assumption:b2}--\ref{assumption:b4}. Then
\[
\|\widehat{\mathscr L}_N-\mathscr L\|_{HS}
=
\mathcal O_P\!\left(N^{-1/2}+a_m^{1/2}\right),
\qquad
\|\widehat{\mathscr L}_N-\mathscr L\|_{HS}^2
=
\mathcal O_P\!\left(\frac{1}{N}+a_m\right).
\]
\end{lemma}

\begin{proof}
Put \(k_P=\kappa(P,\cdot)\) and
\[
T_i=k_{P_i}\otimes k_{P_i},\qquad
\widehat T_i=k_{\widehat P_i}\otimes k_{\widehat P_i}.
\]
The sampling calculation is explicit:
\[
\E\|\mathscr L_N-\mathscr L\|_{HS}^2
=\frac1{N^2}\sum_i\E\|T_i-\mathscr L\|_{HS}^2\le\frac1N,
\]
using independence across bags and \(\|T_i\|_{HS}\le1\).
For the bag perturbation, the exact rank-one decomposition used in Lemma~\ref{lem:op_dev} gives
\[
\|\widehat T_i-T_i\|_{HS}
\le2L_\kappa D(P_i,\widehat P_i)^{\beta_\kappa}.
\]
Therefore
\[
\|\widehat{\mathscr L}_N-\mathscr L_N\|_{HS}
\le\frac{2L_\kappa}{N}\sum_iD(P_i,\widehat P_i)^{\beta_\kappa}
=O_P(\rho(m)^{\beta_\kappa}).
\]
Since \(\beta_\kappa\ge\beta'\) and \(\rho(m)\le1\) eventually,
\[
\|\widehat{\mathscr L}_N-\mathscr L\|_{HS}
=O_P(N^{-1/2}+\rho(m)^{\beta'}),
\]
and squaring yields the sharper intermediate bound
\[
O_P(N^{-1}+\rho(m)^{2\beta'}).
\]
Finally \(\rho(m)^{2\beta'}\le a_m\), and equivalently
\(\rho(m)^{\beta'}\le a_m^{1/2}\), proving the two envelopes in the statement.
The looser common \(a_m\) is retained because the risk-transfer calculation also uses first-order proxy moments.
\end{proof}

The next lemma is the same Davis--Kahan eigenspace perturbation bound as Lemma~\ref{lem:eigenspace-bound}, applied to the distribution-valued covariate setting.

\begin{lemma}\label{lem:dist_dk}
On the gap event in Assumption~\ref{assumption:b5},
\[
\|\widehat\Pi_s^{\mathcal H}-\Pi_s\|_{\mathrm{op}}
\le
\frac{3\|\widehat{\mathscr L}_N-\mathscr L\|_{\mathrm{op}}}{\Delta_s}
\le
\frac{3\|\widehat{\mathscr L}_N-\mathscr L\|_{HS}}{r_s}.
\]
\end{lemma}

\begin{proof}
The proof is identical to Lemma~\ref{lem:eigenspace-bound}, with $X$ and $\widehat X$ replaced by $P$ and $\widehat P$ and Assumption~\ref{assumption:a5} replaced by Assumption~\ref{assumption:b5}.
\end{proof}

\begin{lemma}[Eigenspace contribution to excess risk]\label{lem:dist_subspace}
Assume the bag-wise iid model and Assumptions~\ref{assumption:b2}--\ref{assumption:b6}. Let
\[
g_s=\Pi_sg,
\qquad
\widehat g_s^{\mathcal H}
=
\widehat\Pi_s^{\mathcal H}g.
\]
Then
\[
\|\widehat g_s^{\mathcal H}-g_s\|_{L_2(\Pb_P)}^2
=
\mathcal O_P\!\left\{
\frac{1}{r_s^2}
\left(\frac1N+a_m\right)
\right\}.
\]
\end{lemma}

\begin{proof}
Since
\[
\widehat g_s^{\mathcal H}-g_s
=
(\widehat\Pi_s^{\mathcal H}-\Pi_s)g,
\]
Assumption~\ref{assumption:b6} and $\lambda_1\le a^2$ give
\[
\|\widehat g_s^{\mathcal H}-g_s\|_{L_2(\Pb_P)}^2
\le
 a^2R^2
\|\widehat\Pi_s^{\mathcal H}-\Pi_s\|_{\mathrm{op}}^2.
\]
By Lemma~\ref{lem:dist_dk},
\[
\|\widehat g_s^{\mathcal H}-g_s\|_{L_2(\Pb_P)}^2
\le
\frac{9a^2R^2}{r_s^2}
\|\widehat{\mathscr L}_N-\mathscr L\|_{HS}^2.
\]
Lemma~\ref{lem:dist_op_dev} completes the proof.
\end{proof}

\begin{lemma}[Proxy evaluation transfer]\label{lem:dist_eval_gap}
Assume Assumptions~\ref{assumption:b2} and \ref{assumption:b4}. Let
\[
\mathcal G_B
=
\{h\in\mathcal H_\kappa:\|h\|_{\mathcal H_\kappa}\le B\}.
\]
Then, uniformly over $h\in\mathcal G_B$,
\[
\E\!\left[(h(\widehat P)-h(P))^2\right]
\lesssim
B^2a_m,
\]
and
\[
\left|
\E\!\left[(Y-h(\widehat P))^2\right]
-
\E\!\left[(Y-h(P))^2\right]
\right|
\lesssim
(1+B^2)a_m.
\]
\end{lemma}

\begin{proof}
For every \(h\in\mathcal G_B\), reproduction gives
\[
|h(\widehat P)-h(P)|
\le BL_\kappa D(\widehat P,P)^{\beta_\kappa}.
\]
Moment comparison under \ref{assumption:b4} therefore yields
\[
\E_*[(h(\widehat P)-h(P))^2]
\le B^2L_\kappa^2\E D(\widehat P,P)^{2\beta_\kappa}
\le CB^2\rho(m)^{2\beta_\kappa}\le CB^2a_m.
\]
For squared risks, write their difference as
\[
\E_*[\{h(P)-h(\widehat P)\}\{2Y-h(\widehat P)-h(P)\}].
\]
The two \(L_2\) factors in Cauchy--Schwarz satisfy
\[
\|h(P)-h(\widehat P)\|_{L_2}
\le CB\rho(m)^{\beta_\kappa},
\qquad
\|2Y-h(\widehat P)-h(P)\|_{L_2}\le2\|Y\|_{L_2}+2B.
\]
Their product is at most
\[
CB(\|Y\|_{L_2}+B)\rho(m)^{\beta_\kappa}
\le C(1+B^2)a_m.
\]
This uniform bound is applicable to a fitted function on its deterministic-radius event, for each fixed training realization.
\end{proof}

\begin{lemma}[Estimation in the learned feature space]\label{lem:dist_within}
Let \(Z_{ij}=\widehat\phi_j(\widehat P_i)\), \(G_n=Z^\top Z/n\), and
\[
\Sigma_s=\E_*[\widehat\Phi_s(P)\widehat\Phi_s(P)^\top].
\]
For a bag-covariate-measurable comparator \(h_0=\sum_ja_j\widehat\phi_j\), let
\(r_i=f(P_i)-h_0(\widehat P_i)\).
On \(G_n\succ0,\ \Sigma_s\preceq4G_n\), with conditional probability at least \(1-\delta\),
\[
\begin{aligned}
\|\widehat g_{s,\xi}-h_0\|_{L_2(\Pb_P)}^2
\le48\bigg\{&
16\tau^2\frac{s+\log(1/\delta)}n\\
&+\left\|G_n^{-1/2}\frac{Z^\top r}{n}\right\|_2^2
+\frac{\xi}{4}\|a\|_2^2\bigg\}.
\end{aligned}
\]
Under \eqref{eq:dist_secondstage_suff}, the Gram event has probability tending to one. For the transported comparator \(h_s^0=V_sO_s^\top b_s\), its residual score is
\[
O_P\!\left(s/n+r_s^{-2}(N^{-1}+a_m)\right).
\]
\end{lemma}

\begin{proof}
The finite-design identity, now using the actual observed bags, is
\[
Y=Za+r+\mu,\qquad
(G_n+\xi I)(\widehat\gamma-a)
=Z^\top\mu/n+Z^\top r/n-\xi a.
\]
All five steps of Lemma~\ref{lem:within_space} apply: conditional sub-Gaussian noise is supplied by \ref{assumption:b1}, and the empirical functions remain orthonormal in \(\mathcal H_\kappa\).
In particular the projected noise is bounded by
\(16\tau^2\{s+\log(1/\delta)\}/n\), and
\(\sup_{t\ge0}\xi^2t/(t+\xi)^2=\xi/4\) gives the same ridge term.

For completeness, the inputs to the reused-covariate part are exactly
\[
\begin{array}{c|c}
\text{Generic proof input}&\text{Distribution input}\\ \hline
\text{independent triples and conditional noise}&\ref{assumption:b1}\\
\text{bounded H\"older kernel}&\ref{assumption:b2}\\
\text{spectrum and bounded eigenfunctions}&\ref{assumption:b3}\\
\text{proxy moments}&\ref{assumption:b4}\\
\text{gap and target radius}&\ref{assumption:b5},\ref{assumption:b6}
\end{array}
\]
Thus Technical Lemma~\ref{tl:gram_score} applies to \(P_i,\widehat P_i\) with
\(d_N^2=N^{-1}+a_m\).
Its normalized discrepancy has the concrete bound
\[
\frac1n\sum_i\|w_s(\widehat P_i)-\varphi_s(P_i)\|_2^2
\le\frac C{\lambda_s}
\left\{\frac{\eta_N^2}{r_s^2}
+\frac1n\sum_iD(P_i,\widehat P_i)^{2\beta_\kappa}\right\}
=o_P(1).
\]
The residual decomposition is
\[
f(P_i)-h_s^0(\widehat P_i)
=(f-g_s)(P_i)+(g_s-h_s^0)(P_i)
+h_s^0(P_i)-h_s^0(\widehat P_i).
\]
Its population-feature score is centered, its alignment score is controlled by
\(\E m_n=\mathscr L(g-g_s)\), and its proxy score is at most
\(O_P(\rho(m)^{2\beta_\kappa}/\lambda_s)\).
Since \(r_s^2\le\lambda_s\) and \(\rho(m)^{2\beta_\kappa}\lesssim a_m\), that last quantity is \(O_P(a_m/r_s^2)\).
These are precisely the three estimates in Steps 5--6 of Technical Lemma~\ref{tl:gram_score}, proving the final assertion without treating learned bag features as iid.
\end{proof}

\subsubsection{Proof of Proposition~\ref{prop:dist_R_bound}}

\begin{proof}[Proof of Proposition~\ref{prop:dist_R_bound}]
Work on the intersection of the bag gap event and the Gram event verified in Lemma~\ref{lem:dist_within}; its probability tends to one under \eqref{eq:dist_secondstage_suff}.
The transported comparator \(h_s^0=V_sO_s^\top b_s\) has RKHS norm at most \(R\) and satisfies
\[
\|h_s^0-g\|_{L_2(\Pb_P)}^2
\le2R^2\lambda_{s+1}+CR^2\eta_N^2/r_s^2
=O_P\!\left(s^{-q}+\frac{N^{-1}+a_m}{r_s^2}\right).
\]
Its actual observed-bag residual score is bounded in Lemma~\ref{lem:dist_within}.
Substitute that bound in the finite-design inequality of the same lemma:
\[
\|\widehat g-h_s^0\|_{L_2(\Pb_P)}^2
=O_P\!\left(s/n+\frac{N^{-1}+a_m}{r_s^2}+\xi R^2\right).
\]
Adding twice each of the last two squared errors gives $\mathcal E_P(\widehat g)=O_P(R_{s,n,N,m}^2)$ for the deterministic expression \eqref{eq:R_bound_dist}.
The argument of Technical Lemma~\ref{tl:radius}, with \(P\) in place of \(X\), also proves the stated fitted-radius bound, since \(n^{-1}\sum_iY_i^2\to\E Y^2\).
\end{proof}

\subsubsection{Proof of Theorem~\ref{thm:dist_reg_pred}}

\begin{proof}[Proof of Theorem~\ref{thm:dist_reg_pred}]
\emph{Step 1: latent orthogonality.}
For \(h\in\mathcal H_\kappa\), expanding the population risk at \(g+th\) gives
\[
\E[(f(P)-g(P)-th(P))^2]
=\epsilon^2-2t\E[(f-g)h]+t^2\E h(P)^2.
\]
Since \(t=0\) minimizes this expression for both signs of \(t\),
\(\E[(f-g)h]=0\).
For each fixed training realization take \(h=\widehat g-g\) and use a fresh test \(P\).
The cross term in the squared error vanishes, giving the exact identity
\[
\E_*[(\widehat g(P)-f(P))^2]
=\epsilon^2+\|\widehat g-g\|_{L_2(\Pb_P)}^2.
\]
Proposition~\ref{prop:dist_R_bound} proves the regression bound.

\emph{Step 2: remove fresh response noise.}
Conditional mean zero gives
\[
\E_*[\mu(f(P)-\widehat g(P))]
=\E_*[(f(P)-\widehat g(P))\E(\mu\mid P)]=0.
\]
Expanding the square therefore yields
\[
\mathcal R_P(\widehat g)
=\epsilon^2+\tau_0^2+\|\widehat g-g\|_{L_2(\Pb_P)}^2.
\]

\emph{Step 3: transfer evaluation and add constants explicitly.}
The derived radius event and Lemma~\ref{lem:dist_eval_gap} give
\[
|\mathcal R_{\widehat P}(\widehat g)-\mathcal R_P(\widehat g)|
\le C(1+B_s^2)a_m\le C_2a_m/r_s^2.
\]
Let
\(U=s^{-q}+r_s^{-2}(N^{-1}+a_m)+s/n+\xi\).
On any event where the estimation error is at most \(C_1U\),
\[
\mathcal R_{\widehat P}(\widehat g)
\le\epsilon^2+\tau_0^2+C_1U+C_2a_m/r_s^2
\le\epsilon^2+\tau_0^2+(C_1+C_2)U.
\]
For the matching lower control, use the latent risk identity and the absolute transfer bound:
\[
|\mathcal R_{\widehat P}(\widehat g)-\epsilon^2-\tau_0^2|
\le\mathcal E_P(\widehat g)+|\mathcal R_{\widehat P}(\widehat g)-\mathcal R_P(\widehat g)|
=O_P(R_{s,n,N,m}^2).
\]
This proves the prediction statement.
\end{proof}

\subsubsection{Proof of Corollary~\ref{cor:dist_rate}}

\begin{proof}[Proof of Corollary~\ref{cor:dist_rate}]
Choose a positive integer \(s_n\) neighboring a constant multiple of \(n^{1/(q+1)}\).
As in the derivative calculation for Corollary~\ref{cor:error-rate},
\[
s_n^{-q}\asymp s_n/n\asymp n^{-q/(q+1)}.
\]
The two negligibility conditions make the operator--bag and ridge terms smaller than this order.
Moreover the sufficient second-stage conditions are satisfied at this tuning:
\[
s_n^2/n\to0,\qquad n\lambda_{s_n}\to\infty,\qquad
\frac{N^{-1}+a_m}{\lambda_{s_n}r_{s_n}^2}=o(1).
\]
The last assertion follows by dividing the operator negligibility condition by
\(\lambda_{s_n}\asymp s_n^{-q}+s_n/n\).
Apply Proposition~\ref{prop:dist_R_bound} and then Theorem~\ref{thm:dist_reg_pred}.

For a concrete bag-size interpretation, take the KDE conditions of Technical Lemma~\ref{tl:kde_moments} and \(r_s\gtrsim s^{-\alpha}\).
Since \(a_m=O(m^{-\beta'/(2+d)})\), sufficient negligibility conditions are
\[
N\gg n^{(q+2\alpha)/(q+1)},\qquad
m\gg n^{(2+d)(q+2\alpha)/\{\beta'(q+1)\}}.
\]
These are sufficient conditions for the common risk-transfer envelope; the sharper operator-only moment in Lemma~\ref{lem:dist_op_dev} is not substituted for the entire envelope.
\end{proof}

\section{Experiment Details} \label{append:simulation}
In our notation, $N$ is a total number of examples and $n$ is the number of labeled examples (so we have $N-n$ unlabeled examples in our data). For all simulations, we assume we have $m$ samples for each distribution feature. In other words, our observed data would be represented as the following form.
\begin{equation*}
\begin{gathered}
\mathcal{D}^{obs} = \lbrace (\mathcal{X}_1,y_1),\cdots,(\mathcal{X}_n,y_n),\mathcal{X}_{n+1},\cdots,\mathcal{X}_N \rbrace, \\ 
\text{where} \ \mathcal{X}_i = \lbrace X_{i1}, \cdots. , X_{im} \rbrace \ \text{for} \ i = 1, \cdots, N.
\end{gathered}
\end{equation*}

\subsection{Summary of baseline group methods} \label{appendix:baseline}
We use four groups of baseline methods; three different settings for supervised learning and the other one for semi-supervised learning. The first group is supervised counterpart (\textbf{SL1}) that is able to take distribution features without any unlabeled examples; we use Support Distribution Machines ({SDM}) regression \cite{Sutherland12}. The second supervised group (\textbf{SL2}) uses the first four moments and their second order interactions as features instead of taking distribution features directly; we employ $\text{LASSO}_1$, Random Forest ({RF}), Support Vector Machine regression ({SVR}), and Kernel Ridge Regression ({KRR}). For the third group of our supervised counterpart (\textbf{SL3}), algorithms were implemented upon all the individual samples ($n \times m$) treating each of them as independent example;  we use $\text{LASSO}_2$, Convolutional Neural Network ({CNN}) and {RF}. Finally, for our semi-supervised learning counterpart (\textbf{SSL}) without treating features as a distribution, we employ Laplacian Regularized Least Squares ({LapRLS}) \cite{Belkin06} and {Embed CNN} \cite{Weston08} (E-CNN) which are known to yield the most state-of-the-art performance of semi-supervised learning. For this group, we model the case of both moment based features (as in group 2) and $N \times m$ individual samples (as in group 3), then only present the one with better out-of-sample performance. 

Descriptions of these baseline group methods are also summarized in Table \ref{tb-baseline}. We only include baseline methods that are known to work well in each setting.

\begin{table*}[h]
	\caption{Summary of baseline group methods}
	\label{tb-baseline}
	\begin{center}
		\begin{tabular}{l||c|c|c|}
			\hline
			\thead{Group} & \thead{Methods} & \thead{Distributional \\ Features} & \thead{Unlabeled \\ Examples} \\
			\hline \hline
			\makecell{Supervised \\ Counterpart $\RN{1}$}  & SDM & Yes & No  \\ \hline                  
			\makecell{Supervised \\ Counterpart $\RN{2}$}  & $\text{LASSO}_1$, RF, SVR, KRR & \makecell{No;\\ Use first four moments \\ and their interactions} & No \\ \hline
			
			\makecell{Supervised \\ Counterpart $\RN{3}$}  & $\text{LASSO}_2$, CNN, RF & \makecell{No;\\ Use all the \\ individual samples} & \makecell{No}  \\ \hline
			
			\makecell{Semi-supervised \\ Counterpart} & LapRLS, Embed CNN & \makecell{No;\\ Either moments or \\ individual samples} & Yes \\ \hline
			
			\makecell{SSDR} & \textbf{Proposed estimator} & Yes & Yes  \\ \hline
		\end{tabular}
	\end{center}
\end{table*}

\subsection{Parameter Selection Process} \label{appendix:param.selection}
In this section we describe the parameter selection process for each method we use for the real-data simulation. Throughout the simulation we used zero regularization penalty ($\xi=0$) to avoid excessive computational cost with understanding that the regularization might be of help to reduce overfitting and thus to obtain more stable result.

\begin{itemize}[leftmargin=*]
	\item \textbf{SSDR}: For density estimation, we use RBF kernel with bandwidth chosen from cross-validation. For our distribution kernel, we simply use Euclidean kernel ($\kappa(f,g)=\int f(x)g(x)dx$). To determine the value of $s$, one needs to specify value of $a$ and $q$ in (A3) first, which can be estimated from empirical eigenvalue distribution. In our experiment, we found that in most cases, it is sufficient to use the value of $a \approx 1$ and $q$ slightly larger than 2 to upper bound the $\lambda_k$ for all $k$. Then we select $s$ if marginal error improvement becomes less than $0.01$ within range $10^1 \leq s \leq 10^3$. 
	
	\item \textbf{SDM}: We borrow code from Sutherland $\texttt{https://github.com/dougalsutherland/sdm}$ with default setting for parameters.
	
	\item \textbf{LASSO}: We use R-package $\texttt{glmnet}$ and built-in cross-validation procedure in the package for parameter selection.
	
	\item \textbf{SVR}: We use $\texttt{sklearn.svm}$ in Python and use grid search to determine value of $C$ from $1 \leq C \leq 20$.
	
	\item \textbf{RF}: We use R-package $\texttt{range}$ and built-in cross-validation procedure in the package for parameter selection.
	
	\item \textbf{KRR}: We use $\texttt{sklearn.kernel\_ridge}$ in Python with Gaussian kernel and use grid search for parameter selection.
	
	\item \textbf{CNN}: We use tensorflow library via Python to implement plain 2-layer CNN of \{10, 10\} where only the first layer is convolutional with stride of 2. ADAM optimizer was used for training.
	
	\item \textbf{LapRLS}: We slightly modified the original Matlab code from Belkin and Sindhwani \footnote{\path{http://manifold.cs.uchicago.edu/manifold\_regularization/software.html}}. We use heat kernel weights $W_{ij}=e^{-\Vert x_i-x_j \Vert_2^2/t}$ as our weighted edge for some suitable choice of $t$. By simple grid search, we set $\gamma_Al=0.005, \frac{\gamma_Il}{(u+l)^2}=0.05$ for the first simulation and $\gamma_Al=\frac{\gamma_Il}{(u+l)^2}=0.005$ for the second.  
	
	\item \textbf{Embed CNN}: We added a semi-supervised loss ($l_2$-norm regularizer) to the supervised loss on the entire network???s output based on the 2-layer CNN used as supervised counterpart with the same weighted edge matrix $W$.
	
\end{itemize}

\subsection{Synthetic Experiment: Noisy Euclidean Covariates}
\label{append:noisy.euclidean}

We also include a synthetic experiment for the noisy Euclidean covariate setting in Section~\ref{sec:me_ssl}. This experiment is designed to isolate the effect of unlabeled proxy covariates, separately from the distribution-regression structure studied in Appendix~\ref{append:beta.dist}. We generate a two-dimensional latent variable
\[
U_i=(U_{i1},U_{i2})\sim \mathrm{Unif}([-1,1]^2),
\]
and embed it into a ten-dimensional smooth nonlinear covariate
\[
X_i=
\left(
U_{i1},
U_{i2},
\sin(\pi U_{i1}),
\cos(\pi U_{i2}),
U_{i1}U_{i2},
U_{i1}^2,
U_{i2}^2,
\sin\{\pi(U_{i1}+U_{i2})\},
\cos\{\pi(U_{i1}-U_{i2})\},
U_{i1}^3-U_{i2}^3
\right)^\top\in\mathbb R^{10}.
\]
The response is generated from the smooth regression function
\[
Y_i=f(X_i)+\epsilon_i,
\qquad
f(X_i)=\sin(\pi U_{i1})+U_{i2}^2+\frac{1}{2}U_{i1}U_{i2},
\qquad
\epsilon_i\sim N(0,0.1^2).
\]
The learner does not observe \(X_i\). Instead, it observes the noisy proxy
\[
\widehat X_i=X_i+\tau W_i,
\qquad
W_i\sim N(0,I_{10}),
\]
with \(\tau\in\{0.10,0.40\}\), representing mild and stronger proxy noise.

We fix the number of labeled examples at \(n=50\), vary the total number of proxy covariates \(N\in\{100,300,1000\}\), and evaluate prediction on an independent test sample. We compare three methods: (i) SSDR, which learns spectral features from all \(N\) proxy covariates and fits the second-stage regression using only the labeled observations; (ii) a labeled-only spectral baseline, which uses the same spectral ridge pipeline but learns the features only from the \(n\) labeled proxy covariates; and (iii) an oracle latent baseline, which applies the same procedure using the unobserved latent covariates \(X\). Performance is measured by normalized RMSE on the test set, and the reported standard deviations are computed across repeated Monte Carlo replications.

The results are reported in Table~\ref{tb-noisy-euclidean}. When proxy noise is mild, increasing \(N\) substantially improves SSDR and moves it closer to the oracle latent baseline. When proxy noise is stronger, the improvement from increasing \(N\) is smaller and begins to saturate, consistent with the proxy perturbation term in Theorem~\ref{prop:error-estimator-best}.

\subsection{Synthetic Experiment: Skewness of Beta distribution} \label{append:beta.dist}

To check the validity of our algorithm, we use the Beta distribution in which the first parameter is chosen randomly from certain range of numbers whereas the second parameter is fixed. Specifically, we have $\beta eta(a,3)$ where $a$ is chosen randomly from $\textit{Uniform}[3,20]$. Then sample of size $m=30$ is drawn from $\beta eta(a,3)$ for each $a$. For each setting, the training pool comprises $n$ labeled bags and $N-n$ additional unlabeled bags, with $N-n\in\{0,50,100,150,200,250,300\}$. Thus, $N-n=300$ is the largest additional unlabeled sample size, and test bags are kept separate from this training pool. 

In our case $b$ remains fixed at 3 hence $f$ is a simply function of $a$ which is a distributional feature. In this example we use the kernel $\kappa(f,g)=\int f(x)g(x)dx$, and
calculate the error by using the root mean squared error (RMSE) of the target function at the predicted values of the parameter for the test samples. The results are presented in Table \ref{tb-synthetic}. The results show that in general the prediction error decrease with $N-n$ and $n$ as expected and indeed the algorithm achieves very small prediction errors even for moderate values of $(N,n,m)$.

\begin{table*}[t!]
\vskip -0.1in
	\caption{Prediction Error for different choices of $n$ and $N-n$ for the synthetic dataset of beta distribution}
	\centering
		\fontsize{9}{10}\selectfont
		\begin{tabular}{l lllllll}
			\toprule
			\multirow{2}{*}{$n$} &\multicolumn{5}{l}{\qquad\qquad\qquad\qquad \qquad \qquad $N-n$}\\ 
			\cmidrule{2-8}
			&0&50&100&150&200&250&300\\
            \midrule \midrule
			50&  0.3010 & 0.2780 &  0.2766 & 0.2259 & 0.1560 & 0.1668 & 0.1501\\
			60&  0.2466 & 0.1897 & 0.1543 & 0.1357 & 0.0959 & 0.0860 & 0.0942\\
			70&  0.2051  & 0.1362 &   0.0552  & 0.0546 & 0.0390 & 0.0461 & 0.0444\\
			80&  0.1551  & 0.0962 &   0.0612  & 0.0359 & 0.0346 & 0.0290 & 0.0301\\
			90&  0.1432  &  0.0733  &  0.0526 & 0.0244 & 0.0251& 0.0197 & 0.0153 \\
			\bottomrule
	\end{tabular}
\vskip 0.1in
\label{tb-synthetic}
\normalsize
\vskip -0.25in
\end{table*}

\subsection{Mass Measurements of Galaxy Clusters} \label{appendix:galaxy}
In this real-data application, we estimate mass of galaxy clusters using distribution of galaxy velocity. All the cluster catalog are gleaned from a publicly available halo catalog of the Multidark Simulation\footnote{\path{https://www.cosmosim.org/}}. We randomly choose 80\% of the data as the training pool ($N=1632$ clusters), 10\% for validation whenever necessary for parameter selection, and use the remaining 10\% for testing.

Clusters are the most massive gravitationally-bound systems in the Universe. Since a cluster mass is not directly observable, once clusters are identified mass measurements are needed to map observable cluster properties to the underlying mass. \cite{Ntampaka15} studied the modern machine learning approach for cluster dynamical mass measurements. The power law says Halo mass $M$ can be related to galaxy line-of-sight (LOS) velocity ($v_{los}$) dispersions. Unlike earlier studies that only use summary statistics of mass measurements, they take advantage of using the information encoded in the probability distribution of $v_{los}$ based on the method from \cite{Sutherland12}. It is known that measuring the galaxy velocity is much easier than estimating the galaxy mass, hence it is reasonable to assume we have more unlabeled data in hand. Our algorithm can be well applied to estimate the cluster mass from this distribution variable.

Since the unit of galaxy mass is very huge, we calculate the prediction error as scaled root-mean-squared-error. We randomly select $n$ labeled clusters from the $N=1632$ training-pool clusters and vary the ratio $n/N$; the remaining $N-n$ training-pool clusters are unlabeled. Here, $m_i$ corresponds to the number of galaxies in the $i^{th}$ cluster which is fairly large. In Table \ref{tb-galaxy}, we note that the proposed algorithm (\textbf{SSDR}) generally outperforms other baseline methods. We recognize the prediction error is quite big in this simulation since we only use one distributional variable in our model.

To compute eigenvalues efficiently, we use Lanczos Method. If the algorithm did not converge (which happened only once in this galaxy dataset), we additionally imposed sparsity structure on the kernel matrix and implemented relevant algorithms (e.g. a function found in \texttt{scipy.sparse}).

\subsection{Default Risk and Stock Returns} \label{appendix:default}
A firm defaults when it fails to service its debt obligations. Therefore, default risk induces lenders to require from borrowers a spread over the risk-free rate of interest. This spread is an increasing function of the probability of default of the individual firm. Thus, the relationship between default risk and stock returns has important implications for the risk-reward trade-off in financial markets. Various attempts have been made in financial literature to assess the effect of default risk on equity returns or vice versa.

However, a relationship between stock returns and default risk is still controversial. For example, \cite{garlappi08, GARLAPPI11} support the negative relationship between default risk and stock returns, whereas \cite{Chava10} argue the opposite. Even though many of studies recognize the importance of realized stock return distributions on future default risk, most of them merely rely on running multiple cross-sectional regressions with simple cumulative return or average return for analysis. 

For the simulation we collect stocks' daily returns up to month $T$ with 90-day windows and aim to predict the default risk of time $T+1$ (or $T+2$ in case that the majority of data are not available), and this prediction is rolling over entire time interval (5 years). To compute company's default risk, we collect data from FnGuide \footnote{Dataguide by FnGuide: \url{http://www.dataguide.co.kr/DG5web/eng/index.asp}}. Each vendor adopts slightly different model to estimate company's 1-year default risk basically from CDS spread. In fact, there are significant missing values for default risk at each time point, which motivates us to use the proposed algorithm. 

We use same set of baseline methods as the galaxy cluster dataset. In fact, even though the screening standard is top 500 in company's market value there are quite substantial portion of missing values each time point (roughly 18\% of entire data on average). So composition of labeled samples can be slightly varying from time to time but we do not think this largely affect our analysis. Moreover, many papers in finance literature treat stock returns as \textit{i.i.d.}, so  the use of kernel density estimation to estimate $\hat{p}$ is also justified in this analysis. Since in general the default risk is very small, we use scaled RMSE to measure the prediction error. Again, we confirm our method generally outperforms the other baseline methods.

\end{document}